\documentclass[sigconf]{aamas}
\usepackage[utf8]{inputenc}
\usepackage[T1]{fontenc}
\usepackage{hyperref}
\usepackage{url}
\usepackage{booktabs}
\usepackage{amsfonts}
\usepackage{nicefrac}
\usepackage{microtype}
\usepackage{xcolor}

\usepackage{algorithm}
\usepackage[noend]{algpseudocode}

\usepackage{multirow}
\usepackage{amsthm}
\usepackage{amsmath}
\usepackage{mathtools}

\usepackage{wrapfig}
\usepackage[skip=6pt]{caption}

\newtheorem{theorem}{Theorem}
\newtheorem{corollary}{Corollary}
\newtheorem{lemma}{Lemma}
\newtheorem{proposition}{Proposition}
\newtheorem{definition}{Definition}
\newtheorem{assumption}{Assumption}
\newtheorem*{remark}{Remark}

\usepackage{balance}
\usepackage{etoolbox}

\title{Decentralized Federated Policy Gradient with Byzantine Fault-Tolerance and Provably Fast Convergence}

\author{Philip Jordan}
\affiliation{
  \institution{ETH Zürich}
  \country{Switzerland}}
\email{jordanph@ethz.ch}

\author{Florian Grötschla}
\affiliation{
  \institution{ETH Zürich}
  \country{Switzerland}}
\email{fgroetschla@ethz.ch}

\author{Flint Xiaofeng Fan}
\affiliation{
  \institution{National University of Singapore}
  \country{Singapore}}
\email{fxf@u.nus.edu}

\author{Roger Wattenhofer}
\affiliation{
  \institution{ETH Zürich}
  \country{Switzerland}}
\email{wattenhofer@ethz.ch}

\begin{abstract}
In \emph{Federated Reinforcement Learning (FRL)}, agents aim to collaboratively learn a common task, while each agent is acting in its local environment without exchanging raw trajectories. Existing approaches for FRL either (a) do not provide any fault-tolerance guarantees (against misbehaving agents), or (b) rely on a trusted central agent (a single point of failure) for aggregating updates. We provide the first decentralized Byzantine fault-tolerant FRL method. Towards this end, we first propose a new centralized Byzantine fault-tolerant policy gradient (PG) algorithm that improves over existing methods by relying only on assumptions standard for non-fault-tolerant PG. Then, as our main contribution, we show how a combination of robust aggregation and Byzantine-resilient agreement methods can be leveraged in order to eliminate the need for a trusted central entity. Since our results represent the first sample complexity analysis for Byzantine fault-tolerant decentralized federated non-convex optimization, our technical contributions may be of independent interest. Finally, we corroborate our theoretical results experimentally\footnote{Our code is available at \url{https://github.com/philip-jordan/decentralized-byzantine-RL}.} for common RL environments, demonstrating the speed-up of decentralized federations w.r.t.\ the number of participating agents and resilience against various Byzantine attacks.
\end{abstract}


\newcommand{\R}{\mathbb{R}}
\newcommand{\N}{\mathbb{N}}
\newcommand{\E}{\mathbb{E}}
\newcommand{\pr}{\text{Pr}}

\DeclareMathOperator*{\argmin}{arg\,min}

\begin{document}

\pagestyle{fancy}
\fancyhead{}

\maketitle

\section{Introduction}
Many real-world reinforcement learning (RL) systems consist of a group of agents (e.g.\ a fleet of autonomous vehicles), in which all agents aim to learn the same task, each in its local environment. Since RL models often suffer from poor sample complexity, collaboration is highly desirable. However, as in the autonomous driving example, trajectories of environment interactions may be made up of large amounts of video and sensor data (too large to transfer between agents), and possibly also with privacy restrictions. This motivates the need for distributed algorithms that can leverage the power of collaboration without sharing raw trajectories. In the broader machine learning context, this setting has been widely studied under the name \emph{Federated Learning} (FL) \cite{kairouz2021advances,li2020federated}, and has inspired \emph{Federated Reinforcement Learning} (FRL) \cite{qi2021federated, zhuo2019federated, fan2021fault,fan2023fedhql} as an analogous concept in RL.

Policy gradient methods are among the most popular algorithms in model-free RL. Existing work studies how to generalize such approaches to FRL in a fully trusted setting \cite{jiang2022mdpgt}. In many practical situations, however, there may not be any guarantee on the trustworthiness of information provided by the participating agents, be it due to e.g.\ communication failure, or malicious attempts at trying to prevent the system from learning. Methods that tolerate the presence of some fraction of Byzantine agents have previously been proposed and demonstrated to defend against some attacks in practice \cite{fan2021fault}. As we are going to discuss in related work, so far, FRL algorithms need additional non-standard assumptions regarding gradient estimation noise.

Moreover, a crucial limitation of previous methods for achieving Byzantine fault-tolerance is the need for one trusted party responsible for aggregating updates, filtering out potentially malicious inputs, and broadcasting results back to all participants. Introducing such a single point of failure seems like a high price to pay for achieving Byzantine resilience---and is going against the very idea of a trustless and robust design. Hence we are aiming for a decentralized system, i.e., a system in which Byzantine behavior of \emph{any} participant can be tolerated, as long as only a reasonable number of such bad actors occur simultaneously. Algorithms achieving both Byzantine fault-tolerance and decentralization have previously been proposed for general non-convex optimization \cite{el2021collaborative} but analyzed only w.r.t.\ infinite-time asymptotic behavior. We propose a novel method and give finite-time sample complexity guarantees for decentralized federated PG with Byzantine fault-tolerance.

More concretely, \textbf{our contributions} can be summarized as follows:
\begin{itemize}
\item As a starting point, we provide a new centralized Byzantine fault-tolerant federated PG algorithm \textsc{ByzPG}, and prove competitive sample complexity guarantees under assumptions that are standard in non-fault-tolerant PG literature. In particular, unlike previous approaches, we do not rely on deterministic bounds on gradient estimation noise.
\item For our main contribution, we extend the above (centralized) approach to the significantly more challenging decentralized setting: With \textsc{DecByzPG}, we propose a decentralized Byzantine fault-tolerant PG method. To the best of our knowledge, this is the first decentralized Byzantine fault-tolerant algorithm for non-convex optimization with sample complexity guarantee.
\item Technically, we leverage the favorable interplay of fault-tolerant aggregation and agreement mechanisms that so far have only been studied in separation. Key to our analysis is a novel lemma on the concentration of random parameter vectors that helps control the bias incurred from agreement.
\item To corroborate our theoretical results regarding both \textsc{ByzPG} and \textsc{DecByzPG} experimentally, we demonstrate speed-up as more agents join a federation, as well as tolerance against different Byzantine attack scenarios.
\end{itemize}

The rest of this paper is organized as follows: In Section \ref{sec:bg}, we provide the necessary background on RL while covering related work on FRL and fault-tolerance. Section \ref{sec:setup} introduces our setup regarding communication, the Byzantine failure model, and technical assumptions needed for the sample complexity analysis. As a warm-up, in Section \ref{sec:ByzPG}, we introduce our centralized algorithm \textsc{ByzPG} which in Section \ref{sec:DecByzPG} is generalized to the decentralized \textsc{DecByzPG}. Finally, our experiments are described and evaluated in Section \ref{sec:exp}.

\section{Background \& Related work}
\label{sec:bg}

\textbf{Reinforcement learning (RL) and policy gradient (PG).} The RL setup is commonly modeled as a Markov Decision Process (MDP, see also \cite{sutton2018reinforcement}) $\mathcal M:= \{\mathcal{S}, \mathcal{A}, \mathcal{P}, \mathcal{R}, \gamma, \rho\}$ with state space $\mathcal{S}$, action space $\mathcal{A}$, transition dynamics $\mathcal{P}(s^\prime \mid s, a)$, and reward $\mathcal{R}: \mathcal{S} \times \mathcal{A} \mapsto[0, R]$ where $R > 0$, $\gamma \in (0,1)$ is the discount factor, and $\rho$ is the initial state distribution.
Let $\pi$ denote the policy controlling an agent's behavior, i.e., $\pi(a \mid s)$ is the probability that the agent chooses action $a$ at state $s$.
A trajectory $\tau:= \{s_0, a_0, \ldots, s_{H-1}, a_{H-1}\}$ is a sequence of state-action pairs followed by an agent according to a stationary policy $\pi$, where $s_0 \sim \rho$. We define $\mathcal{R}(\tau) :=\sum_{t=0}^{H-1} \gamma^{t} \mathcal{R}\left(s_{t}, a_{t}\right)$ as the cumulative discounted reward for a trajectory $\tau$. Note that here we study episodic MDPs with fixed trajectory horizon $H$.

PG methods are popular in model-free RL \cite{schulman2015TRPO,schulman2017PPO}.
Compared to deterministic value-function-based methods, PG is more effective when applied to tasks with high-dimensional or infinite action spaces. Let $\pi_{\theta}$ denote the policy parameterized by $\theta \in \mathbb{R}^d$, and $p(\tau \mid \pi_{\theta})$ the trajectory distribution induced by policy $\pi_{\theta}$. The expected discounted future reward when following $\pi_\theta$ is given by~$J(\theta) := \E_{\tau \sim p(\cdot \mid \theta)}\left[\mathcal{R}(\tau) \mid \mathcal M \right]$ whose gradient w.r.t.\ $\theta$ is
\begin{align*}
\nabla_{\theta} J(\theta)&=\int_{\tau} \mathcal{R}(\tau) \nabla_{\theta} p(\tau \mid \theta) \,d\tau \\
&=\mathbb{E}_{\tau \sim p(\cdot \mid \theta)}\left[\nabla_{\theta} \log p(\tau \mid \theta) \mathcal{R}(\tau) \mid \mathcal M\right].
\stepcounter{equation}\tag{\theequation}\label{def-full-gradient}
\end{align*}
Hence we can use gradient ascent in order to optimize $J(\theta)$ over $\theta \in \R^d$.
Since \eqref{def-full-gradient} involves computing an integral over all possible trajectories, we typically use stochastic gradient ascent. In each iteration, a batch of trajectories $\{\tau_i\}_{i=1}^M$ is sampled at the current policy $\theta$. Then, the policy is updated by $\theta \leftarrow \theta + \eta \widehat{\nabla}_{M} J\left(\theta\right)$, where $\eta$ is the step size and $\widehat{\nabla}_{M} J(\theta)$ is an estimate of \eqref{def-full-gradient} based on the sampled trajectories $\{\tau_i\}_{i=1}^M$: $\widehat{\nabla}_{M} J(\theta)=\frac{1}{M} \sum_{i=1}^{M} \nabla_{\theta} \log p\left(\tau_{i} \mid \theta \right) \mathcal{R}\left(\tau_{i}\right)$.
Commonly used policy gradient estimators, e.g.\ REINFORCE \cite{williams1992REINFORCE} and GPOMDP \cite{baxter2001GPOMDP}, can be written as
\begin{align*}
\widehat{\nabla}_{M} J(\theta)=\frac{1}{M} \sum_{i=1}^{M} g(\tau_i \mid \theta)
\end{align*}
where $\tau_{i}=\{s_0^i, a_0^i, \ldots, s_{H-1}^{i}, a_{H-1}^{i}\}$ and $g(\tau_i \mid \theta)$ is an unbiased estimate of $\nabla_{\theta} \log p(\tau_{i} \mid  \theta) \mathcal{R}(\tau_{i})$. For more details on gradient estimation and sampling, we refer to Appendix \ref{app:sampling}.

\textbf{Non-convex optimization.} Despite PG's additional challenges of non-stationarity and the non-finite-sum structure, improvements in convergence results in non-convex optimization have generally led to similar progress for optimizing the non-concave $J(\theta)$ in PG. The $\mathcal O( \varepsilon^{-4})$ sample complexity for reaching an $\varepsilon$-stationary point, i.e., $\theta$ such that $\mathbb{E}[\|\nabla J(\theta)\|^2] \leq \epsilon^2$, of SGD \cite{reddi2016stochastic} and vanilla PG \cite{papini2021safe} has been lowered to $\mathcal O( \varepsilon^{-10/3})$ by SVRG, and SVRPG \cite{xu2020improved} respectively. These methods rely on an inner loop that reuses old gradient estimates for reduced variance which in the case of SVRPG is implemented via importance sampling. The recently proposed PAGE estimator \cite{li2021page}, and its PG adaptation PAGE-PG \cite{gargiani2022page}, replace the inner loop by a probabilistic switch, lowering the sample complexity to $\mathcal O( \varepsilon^{-3})$.

\textbf{Fault-tolerance.} Byzantine fault-tolerance \cite{lamport2019byzantine} has long been established as the strongest notion of resilience against arbitrary failure or deliberate manipulation of distributed systems. Regarding previous work in the federated optimization literature, we distinguish between the rather common \emph{centralized}, and the far less studied \emph{decentralized}, sometimes called collaborative, setting.
\begin{enumerate}
\item[(a)] \emph{Centralized:} In the presence of a trusted coordinator, Byzantine fault-tolerant non-convex optimization has been widely studied---with approaches differing mostly in terms of filtering techniques and problem assumptions \cite{blanchard2017machine-Krum,alistarh2018byzantine,allen2020byzantine-iclr}. We refer to \cite{farhadkhani2022byzantine} for an overview of such commonly used Byzantine-resilient methods for aggregating potentially malicious updates at a central server. Regarding Byzantine-tolerant PG, \cite{fan2021fault} shows promising empirical results. However, theoretical guarantees are proven only under deterministic noise bounds which makes results difficult to appreciate in comparison with non-fault-tolerant methods that do not rely on such assumptions. Recently, \cite{gorbunov2023variance} has proposed a non-convex optimization algorithm leveraging the favorable interplay of certain robust aggregators and the above-mentioned variance-reduced PAGE estimator. Our centralized \textsc{ByzPG} extends their ideas into the PG setting, with a modified algorithm and tightened analysis.
\item[(b)] \emph{Decentralized:} In the PG context, there is no previous work studying decentralized Byzantine-tolerant methods. More generally, \cite{el2021collaborative} proposes a fault-tolerant algorithm for decentralized non-convex optimization. While convergence is only proven in an infinite-time asymptotic sense, their notion of averaging agreement is shown to be of crucial importance for decentralized learning. Indeed, the notion of $\epsilon$-approximate Byzantine agreement on $d$-dimensional inputs had previously been proposed \cite{mendes2013multidimensional,mendes2015multidimensional}. However, unlike averaging agreement, such methods show poor applicability in our setting since the fraction of tolerable Byzantines goes to zero as $d$ increases.
\end{enumerate}

\section{Setup and Assumptions}
\label{sec:setup}
\subsection{Distributed Computing Setup}
Communication is assumed to happen in a round-based, synchronous, all-to-all manner among $K$ agents, and the exchange of raw trajectories is prohibited. In particular, our algorithms will only involve sending current values of local policy parameters and respective gradients.

In order to model both system failure as well as malicious agent behavior, we tolerate a fraction of Byzantine agents, in particular:
\begin{assumption}[Byzantine agents]
\label{ass:byz-fraction}
Let $\alpha_{\max}=1/2$ in the centralized setting, and $\alpha_{\max}=1/4$ in the decentralized setting, respectively. Denote by $\mathcal H_t \subset [K]$ the set of honest (i.e.\ non-Byzantine) agents in iteration $t$ of an algorithm. Then there exist $\alpha,\bar\epsilon>0$ such that $\bar\alpha:=\alpha + \bar\epsilon < \alpha_{\max}$ and for all $t$, $\mathcal |\mathcal H_t| \geq (1-\alpha)K$.
\end{assumption}

We point out that $\mathcal H_t$ may be different for each iteration $t$, hence it is of no use for any agent to remember past communication in order to infer who might be Byzantine.

Instead of sending updates as prescribed by our algorithms, Byzantine agents may send arbitrary values. In particular, these values may be chosen by an omniscient entity with access to all information (e.g.\ agents' local state, messages that have been sent, the definition of the algorithm, who is Byzantine, etc.) and controlling all Byzantine agents. This means Byzantine agents may collude or base their behavior on any other non-public information. However, Byzantine agents are not omnipotent, e.g.\ they cannot interfere in communication between honest agents by changing or delaying messages. Moreover, we assume that Byzantines cannot alter local state, not even their own state. In the centralized case, this assumption does not change anything, since our algorithm \textsc{ByzPG} only maintains cross-iteration state at the trusted central agent. In the decentralized case, however, corrupted local state may otherwise be passed on from a Byzantine agent to an honest agent across iterations.
Note also that in particular, any agent not sending messages in the required format or omitting updates, potentially due to failure of the communication network, can be modeled as Byzantine.

\subsection{Reinforcement Learning Assumptions}

Our theoretical analysis aims to bound the required number of sampled trajectories required per agent in order to reach an $\epsilon$-stationary solution. In the centralized case, this refers to the central agent finding $\theta \in \R^d$ such that $\|\nabla J ( \theta ) \| \leq \epsilon$ which can then be broadcast to all participants. A generalized solution concept for the decentralized setting is presented in Section \ref{sec:DecByzPG}.

In the following, we state the set of assumptions our analysis is based on, which is standard in the study of PG, see e.g.\ \cite{papini2018stochastic,xu2020improved,yuan2020stochastic,gargiani2022page}. In particular, we do not require a more restrictive version of Assumption \ref{ass:fin-var} made in \cite{fan2021fault}. Hence, our sample complexity results are amenable to comparison with non-fault tolerant counterparts.

Note that we are assuming homogeneity of all agents' local environments, and all agents hence share the same objective $J(\cdot)$.
\begin{assumption}[Log-policy gradient norm]\label{ass:bounded-policy}
For any $a \in \mathcal{A}$ and $s \in \mathcal{S}$, there exists a constant $G>0$ such that for any $\theta \in \mathbb{R}^d$ we have $\left\|\nabla_\theta \log \pi_\theta(a \mid s)\right\| \leq G$.
\end{assumption}

\begin{assumption}[Log-policy smoothness]\label{ass:smoothness}
For any $\theta \in \mathbb{R}^d$, $\pi_\theta$ is twice differentiable, and for any $a \in \mathcal{A}$ and $s \in \mathcal{S}$, there exists a constant $M>0$ such that $\left\|\nabla_\theta^2 \log \pi_\theta(a \mid s)\right\| \leq M$.
\end{assumption}

\begin{assumption}[Gradient estimator variance]\label{ass:fin-var}
There exists a constant $\sigma>0$ such that for any $\theta \in \mathbb{R}^d$, we have $\operatorname{Var}\left[g(\tau \mid \theta)\right] = \mathbb{E}\|g(\tau\mid\theta) - \nabla J(\theta)\|^2 \leq \sigma^2$.
\end{assumption}
\begin{assumption}[Importance weight variance]\label{ass:fin-imp-var}
For any policy pair $\theta_a, \theta_b \in \mathbb{R}^d$ and $\tau \sim p\left(\cdot \mid \theta_b\right)$, the importance weight $\omega\left(\tau \mid \theta_b, \theta_a\right)=\frac{p\left(\tau \mid \theta_a\right)}{p\left(\tau \mid \theta_b\right)}$ is well-defined. In addition, there exists a constant $W>0$ such that $\operatorname{Var}\left[\omega\left(\tau \mid \theta_b, \theta_a\right)\right] \leq W \| \theta_a-\theta_b \|^2$.
\end{assumption}
For completeness, we restate the following commonly used proposition from \cite{xu2020improved}.
\begin{proposition}\label{prop:smoothness}
Under the above assumptions \ref{ass:bounded-policy}, \ref{ass:smoothness}, \ref{ass:fin-var}, and \ref{ass:fin-imp-var}, with $g(\tau\mid\theta)$ denoting the REINFORCE or GPOMDP gradient estimator, we have for all $\theta,\theta_1,\theta_2 \in \R^d$:
\begin{enumerate}
\item $\left\| g_k(\tau \mid \theta) \right\| \leq C_g$ with $C_g=HG(R+|C_b|)/(1-\gamma)$ and $C_b$ is the baseline reward,
\item $\left\| g(\tau\mid\theta_1)-g(\tau\mid\theta_2) \right\| \leq L_g \left\| \theta_1-\theta_2 \right\|$ with $L_g=HM(R+|C_b|)/(1-\gamma)$, and
\item $J(\theta)$ is $L$-smooth with $L=HR(M+ HG^2)/(1-\gamma)$.
\end{enumerate}
\end{proposition}

\section{Centralized Byzantine-tolerant federated PG}
\label{sec:ByzPG}

In this section, we describe \textsc{ByzPG}, given by Algorithm \ref{alg:ByzPG}, our centralized method for Byzantine fault-tolerant PG. This also serves as a warm-up for introducing parts of our method that are going to reappear in Section \ref{sec:DecByzPG}. Note that in Algorithm~\ref{alg:ByzPG}, $\operatorname{Be} \left( p \right)$ denotes a Bernoulli distribution with success probability $p$ and $\mathcal U(S)$ denotes a uniform distribution over a finite set $S$.

\begin{algorithm}[ht]
\caption{\textsc{ByzPG} at server agent}
\label{alg:ByzPG}
\begin{algorithmic}[1]
\State \textbf{input:} $\theta_0 \in \mathbb R^d$, large batch size $N$, small batch size $B$, step size $\eta$, probability $p \in (0,1]$
\For{$t=0$ {\bfseries to} $T-1$}
\State $c_t \leftarrow \operatorname{\textbf{Sample}} \left( \operatorname{Be} \left( p \right) \right)$ \label{line:sample}
\If{$c_t=1$ or $t=0$}
\For{worker agent $k \in [K]$ in parallel}
\State sample trajectories $\{\tau_{t,i}^{(k)}\}_{i=1}^{N}$ from $p(\cdot \mid \theta_{t})$
\State $\widetilde{v}_t^{(k)}=\frac{1}{N}\sum_{i=1}^N g(\tau_{t,i}^{(k)} \mid \theta_t)$
\EndFor
\State $v_t \leftarrow \operatorname{\textbf{Aggregate}}\left( \langle \widetilde{v}_t^{(k)} \rangle_{k=1}^{K} \right)$
\Comment{$\widetilde{v}_t^{(k)}$ received from worker agent $k$, $\forall k \in [K]$}
\Else
\State sample trajectories $\{\tau_{t,i}\}_{i=1}^{B}$ from $p(\cdot \mid \theta_{t})$
\State $v_t=\frac{1}{B}\sum_{i=1}^B g(\tau_{t,i} \mid \theta_t )+v_{t-1}-\frac{1}{B}\sum_{i=1}^B g^{\omega_{\theta_t}}\left( \tau_{t,i} \mid \theta_{t-1} \right)$
\EndIf
\State $\theta_{t+1}=\theta_t+\eta v_t$
\Comment{broadcast $\theta_{t+1}$ to worker agents}
\EndFor
\State \textbf{output:} $\theta_{\widehat{T}}$ with $\widehat{T} \sim \mathcal U \left([T]\right)$
\end{algorithmic}
\end{algorithm}

\subsection{Method}
Instead of the usual inner loop seen in variance-reduced methods such as SVRPG \cite{papini2018stochastic}, we probabilistically switch between update types, as in PAGE-PG \cite{gargiani2022page}. Concretely, in each iteration, we either (a) sample a large batch of $N$ trajectories at $\theta_t$ for gradient estimation, or (b) sample a small batch of $B$ trajectories and use a variance-reduced estimator incorporating the previous iteration's gradient estimate, and employ importance sampling to correct for non-stationarity. For details on gradient estimation, importance sampling, and the definition of $g^{\omega_{\theta_t}}\left( \tau_{t,i} \mid \theta_{t-1} \right)$ we refer to Appendix \ref{app:sampling}. Note that with probability $p$ we use (a), and (b) otherwise---except for the first iteration, in which only (a) is well-defined. Furthermore, (a) is performed at all worker agents in parallel which have previously received the current parameters $\theta_t$ from the server agent. Then, the individual estimates $\widetilde{v}_t^k$ are aggregated at the server via the \textbf{Aggregate} subroutine described below. In the case of (b), we sample and estimate the gradient only at the server, hence there is no need for aggregation.

We point out that loopless variance-reduction has previously been used in conjunction with Byzantine fault-tolerant aggregation in \cite{gorbunov2023variance}'s \textsc{Byz-VR-Marina}. However, \textsc{ByzPG} distinguishes itself in two ways:
\begin{enumerate}
\item \textsc{ByzPG} only samples at workers in case (a) while \textsc{Byz-VR-Marina} does so in either case. Our analysis suggests that unlike in case (a), in (b), the bias introduced by Byzantine filtering outweighs the benefits from the reduced variance of the aggregated sample. In our PG setting, this modification is key to achieving sample complexity competitive with non-fault-tolerant methods.
\item \textsc{Byz-VR-Marina} is designed for general non-convex finite-sum optimization. \textsc{ByzPG} handles the additional challenges of non-stationarity and not having access to the full gradient by relying on importance sampling and switching between a large and small batch size.
\end{enumerate}

The following notion of robust aggregation specifying our requirements on the \textbf{Aggregate} subroutine is adopted from \cite{gorbunov2023variance} and has first appeared in a similar form in \cite{karimireddy2021learning}.

\begin{definition}[robust aggregation]\label{def:aggregate}
Let $C_{ra} > 0$ and $\alpha \in [0,1/2)$. A function $\operatorname{\textbf{Aggregate}}:(\mathbb R^d)^K \to \mathbb R^d$ is an $(\alpha,C_{ra})$-robust aggregator if for any tuple of inputs $\langle \theta^{(k)} \rangle_{k \in [K]}$ with $\theta^{(k)} \in \mathbb R^d$, and for any $\mathcal H \subseteq K$ with $|\mathcal H| \geq (1-\alpha)K$, denoting $\hat{\theta} := \operatorname{\textbf{Aggregate}}(\theta^{(1)},\dots,\theta^{(K)})$, it holds that
\begin{align*}
\mathbb E \left[ \| \hat{\theta}-\bar{\theta} \|^2 \right] \leq \frac{C_{ra} \alpha}{|\mathcal H|(|\mathcal H|-1)} \sum_{i,l \in \mathcal H} \mathbb E \left[ \| \theta^{(i)}-\theta^{(l)} \|^2 \right]
\end{align*}
where $\bar{\theta}:=\frac{1}{|\mathcal H|}\sum_{i \in \mathcal H} \theta^{(i)}$. Expectations are taken over the randomness of the input.
\end{definition}
Known implementations satisfying Definition \ref{def:aggregate} are discussed in Appendix \ref{app:impl-aggregation}. In particular, there exist $(\alpha,C_{ra})$-robust aggregators for constant $C_{ra}$ and any $\alpha \in [0,1/2)$.

\subsection{Convergence Analysis and Sample Complexity}
Next, we present the convergence guarantees for \textsc{ByzPG}, with proofs deferred to Appendix \ref{app:ByzPG-proofs}.
\begin{theorem}\label{thm:ByzPG}
Let Assumptions \ref{ass:bounded-policy}, \ref{ass:smoothness}, \ref{ass:fin-var}, and \ref{ass:fin-imp-var} hold.
Suppose \textbf{Aggregate} is an $(\alpha,C_{ra})$-robust aggregator with constant $C_{ra} > 0$ and $\alpha$ satisfying Assumption \ref{ass:byz-fraction}. Then the following holds for the output of \textsc{ByzPG}, i.e., Algorithm \ref{alg:ByzPG}: For $\eta = \Theta(\min\{ \sqrt{pK},1/L\})$, there exists a constant $C>0$, such that for any $T \geq 1$,
\begin{align*}
\E \left[\| \nabla J \left( \theta_{\widehat{T}} \right) \|^2 \right] \leq \frac{2 \E \left[ \Phi_0 \right]}{\eta T} + \frac{C \sigma^2}{N} \left( \alpha+\frac{1}{K} \right)
\end{align*}
with $\Phi_0 := J^{*} - J \left( \theta_0 \right) + \frac{\eta}{p} \left\| v_0 - \nabla J \left( \theta_0 \right) \right\|^2$ and $J^{*}:=\max_{\theta \in \R^d} J(\theta)$.
\end{theorem}
\begin{corollary}\label{cor:ByzPG}
In the setting of Theorem \ref{thm:ByzPG}, by choosing $p=1/N$, the expected number of trajectories that need to be sampled per agent to achieve $\E[ \| \nabla J \left( \theta_{\widehat{T}} \right) \|^2] \leq \epsilon^2$ is
\begin{align*}
\mathcal O \left( \frac{\alpha^{1/2}}{K^{1/2}\epsilon^3}+\frac{1}{K \epsilon^3} \right).
\end{align*}
\end{corollary}
Observe that in particular, if $\alpha=0$, we need $\mathcal O(K^{-1} \epsilon^{-3})$ trajectories in expectation, and for constant $\alpha > 0$, we need $\mathcal O(K^{-1/2}\epsilon^{-3})$ trajectories in expectation. We hence recover the SOTA sample complexity of PAGE-PG \cite{gargiani2022page} (which is proven under assumptions equivalent to ours) for $K=1$, and asymptotically improve for larger $K$, despite the presence of Byzantines.

\section{Decentralized Byzantine-tolerant federated PG}\label{sec:DecByzPG}
\begin{algorithm*}
\caption{\textsc{DecByzPG} at the $k$-th agent}
\label{alg:DecByzPG}
\begin{algorithmic}[1]
\State \textbf{input:} $\theta_0 \in \mathbb R^d$, large batch size $N$, small batch size $B$, step size $\eta$, probability $p \in (0,1]$
\State initialize $\theta_0^{(k)}=\theta_{0}$
\For{$t=0$ {\bfseries to} $T-1$}
\State $c_t \leftarrow \operatorname{\textbf{Common-Sample}} \left( \operatorname{Be} \left( p \right) \right)$
\State sample trajectories $\{\tau_{t,i}^{(k)}\}_{i=1}^{M}$ from $p(\cdot \mid \theta_{t}^{(k)})$ where $M = \begin{cases}N &\text{ if $c_t=1$ or $t=0$}\\B & \text{ else}\end{cases}$
\vspace{4pt}
\State $\widetilde{v}_t^{(k)}=\begin{cases}
\frac{1}{N} \sum_{i=1}^N g(\tau_{t,i}^{(k)} \mid \theta_t^{(k)}) \quad\quad\quad\quad\quad\quad\quad\quad\quad\quad\quad\quad\quad\quad\quad\;\;\text{if $c_t=1$ or $t=0$}  \\[4pt]
\frac{1}{B} \sum_{i=1}^B g(\tau_{t,i}^{(k)} \mid \theta_t^{(k)})+\frac{1}{\eta}(\theta_{t}^{(k)}-\theta_{t-1}^{(k)})-\frac{1}{B} \sum_{i=1}^B g^{\omega_{\theta_t^{(k)}}}( \tau_{t,i}^{(k)} \mid \theta^{(k)}_{t-1}) \quad\text { else}
\end{cases}$
\vspace{4pt}
\State $v_t^{(k)} \leftarrow \operatorname{\textbf{Aggregate}}\left( \langle \widetilde{v}_t^{(k')} \rangle_{k'=1}^{K} \right)$
\State $\widetilde{\theta}^{(k)}_{t+1}=\theta_t^{(k)}+\eta v_t^{(k)}$
\vspace{2pt}
\State $\theta_{t+1}^{(k)} \leftarrow \textbf{Avg-Agree}_{\kappa}\left( \langle \widetilde{\theta}_{t+1}^{(k')} \rangle_{k'=1}^K \right)$
\EndFor
\vspace{4pt}
\State \textbf{output:} $\theta_{\widehat{T}}^{(k)}$ with $\widehat{T} \sim \textbf{Common-Sample}\left( \mathcal U \left( [T] \right) \right)$
\end{algorithmic}
\end{algorithm*}

\subsection{Method}
In the decentralized setting, instead of having a centrally maintained $\theta_t \in \R^d$, the state at each iteration $t$ is given by a tuple $\langle \theta^{(k)}_t \rangle_{k \in [K]}$ of each agent's local parameters with $\theta^{(k)}_t \in \R^d$. We are interested in the following solution concept.

\begin{definition}[$K$-agent $\alpha$-tolerant $\epsilon$-approximate solution]
\label{def:solution}
For $\epsilon>0$, we call $\langle \theta^{(k)} \rangle_{k \in [K]}$ with $\theta^{(k)} \in \R^d$ a $K$-agent $\alpha$-tolerant $\epsilon$-stationary point if $\exists \mathcal G \subset [K]$ such that $|\mathcal G| \geq (1-\alpha)K$ and $\forall k \in \mathcal G$, we have $\|\nabla J ( \theta^{(k)} ) \| \leq \epsilon$. We say a decentralized algorithm achieves a $K$-agent $\alpha$-tolerant $\epsilon$-approximate solution in $T$ rounds if $\exists \mathcal G_T \subset [K]$ such that $|\mathcal G_T| \geq (1-\alpha)K$ and $\forall k \in \mathcal G_T$, we have $\mathbb{E} [ \|\nabla J ( \theta_T^{(k)} ) \|^2 ] \leq \epsilon^2$, where $\theta_T^{(k)}$ is the output of agent $k$ after $T$ rounds and the expectation is taken w.r.t.\ all randomness of the algorithm.
\end{definition}

As a first step towards decentralizing \textsc{ByzPG}, suppose all agents simultaneously execute \textsc{ByzPG}, each with itself in the role of the server, and denote the $k$-th agent's resulting local parameters in iteration $t$ by $\tilde{\theta}^{(k)}_t$. Since Byzantines may send inconsistent gradient estimates to different agents, already after the first iteration, we may have $\widetilde{\theta}_1^{(k)}\not=\widetilde{\theta}_1^{(k')}$ for $k \not= k'$. Such disagreement on parameters across agents may be detrimental to convergence at each agent.
As a remedy, we adopt the notion of averaging agreement that has been proposed by \cite{el2021collaborative} in the context of Byzantine fault-tolerant collaborative learning.

\begin{definition}[Averaging Agreement]\label{def:agree}
Let \textbf{Avg-Agree}$_{\kappa}$ be a decentralized algorithm that as input receives $\langle \theta^{(k)} \rangle_{k \in [K]}$ where $\theta^{(k)} \in \mathbb R^d$ is known only to agent $k$. Under Assumption \ref{ass:byz-fraction}, let $\mathcal G_t \subseteq \mathcal H_t$ be such that $|\mathcal G_t| \geq (1-\bar\alpha)K$. Suppose after $\kappa$ rounds of communication, where $\kappa \in \mathbb N$ is a parameter of the algorithm, \textbf{Avg-Agree}$_{\kappa}$ terminates with output $\langle \hat{\theta}^{(k)} \rangle_{k \in \mathcal G_t}$ in the form of $\hat{\theta}^{(k)} \in \mathbb R^d$ being known to agent $k$. Then, we say \textbf{Avg-Agree}$_{\kappa}$ achieves $C_{avg}$-averaging agreement for some $C_{avg}>0$, if for any input it is guaranteed that
\begin{align*}
\max_{i,l \in \mathcal G_t} \| \hat{\theta}^{(i)}-\hat{\theta}^{(l)} \| &\leq \frac{\max_{i,l \in \mathcal G_t} \| \theta^{(i)}-\theta^{(l)} \|}{2^\kappa} \quad\text{ and} \\
\| \bar{\hat{\theta}} - \bar{\theta} \| &\leq C_{avg} \cdot \max_{i,l \in \mathcal G_t} \| \theta^{(i)}-\theta^{(l)} \|
\end{align*}
where $\bar{\theta}=\frac{1}{|\mathcal G_t|}\sum_{k \in \mathcal G_t} \theta^{(k)}$ and $\bar{\hat{\theta}}=\frac{1}{\left| \mathcal G_t \right|}\sum_{k \in \mathcal G_t} \hat{\theta}^{(k)}$.
\end{definition}

Known implementations satisfying the above definition are stated and discussed in Appendix \ref{app:impl-agreement}. Our algorithm \textsc{DecByzPG}, as described by Algorithm \ref{alg:DecByzPG}, employs an \textbf{Avg-Agree}$_\kappa$ subroutine at the end of each iteration to ensure averaging agreement on agents' local parameters.

We point out that while \cite{el2021collaborative} makes use of averaging agreement in a similar context, their analysis does not yield sample complexity results. Our improved results rely upon the following insights:
\begin{enumerate}
\item Careful analysis of bias and variance of the \emph{realized} gradient estimates, which we define as $\hat{v}^{(k)}_t:=\frac{1}{\eta} \left( \theta_{t+1}^{(k)}-\theta_t^{(k)} \right)$, reveal that variance-reduced methods combined with the notion of robust aggregation from Definition \ref{def:aggregate} show favorable interplay with averaging agreement. In particular, the low variance of intermediate estimates $\widetilde{v}^{(k)}_t$ and $v^{(k)}_t$ keep the bias introduced by \textbf{Avg-Agree}$_\kappa$ small.
\item Controlling this bias introduced by \textbf{Avg-Agree}$_\kappa$ further requires a bound on the expected diameter of agents' parameters before agreement, i.e., the $\tilde{\theta}^{(k)}_{t+1}$'s. We leverage the fact that only the diameter of some large subset of parameters needs to be bounded, allowing us to apply strong concentration bounds instead of a weak union bound.
\end{enumerate}

In place of \textbf{Sample} in Line \ref{line:sample} of Algorithm \ref{alg:ByzPG}, \textsc{DecByzPG} requires a distributed Byzantine fault-tolerant sampling procedure. While such implementations have been studied in theory \cite{cachin2000random}, in practice, we may simply use a pseudorandom generator with a seed derived from the common initialization $\theta_0$.

\subsection{Convergence Analysis and Sample Complexity}
We next present sample complexity guarantees for \textsc{DecByzPG}, and provide a proof sketch outlining key ideas required for the analysis.
\begin{theorem}\label{thm:DecByzPG}
Let Assumptions \ref{ass:bounded-policy}, \ref{ass:smoothness}, \ref{ass:fin-var}, and \ref{ass:fin-imp-var} hold.
Suppose \textbf{Aggregate} is an $(\alpha,C_{ra})$-robust aggregator for constant $C_{ra} > 0$ and $\alpha$ as in Assumption \ref{ass:byz-fraction}. Let further \textbf{Avg-Agree}$_\kappa$ achieve $C_{avg}$-averaging agreement for constant $C_{avg}>0$. For $A=\Theta \left( \frac{\alpha}{p^2}+\frac{1}{pK} \right)$, choose $\eta = \frac{1}{2}\min \left\{ \frac{1}{\sqrt{A}}, \frac{1}{L} \right\}$, and $\kappa=\Theta \left( \log\frac{NK}{p^2} \right)$. Then the following holds for the output of \textsc{DecByzPG}, i.e., Algorithm \ref{alg:DecByzPG}: There exists a constant $C>0$ such that for any $T \geq 1$, $\exists \mathcal G_{\widehat{T}} \subset [K]$ with $|\mathcal G_{\widehat{T}}| \geq (1-\bar\alpha)K$ and $\forall k \in \mathcal G_{\widehat{T}}$,
\begin{align*}
\E \left[ \left\| \nabla J \left( \theta^{(k)}_{\widehat{T}} \right) \right\|^2 \right] \leq \frac{4 \E \left[ \Phi_0 \right]}{\eta T} + \frac{C\sigma^2}{N} \left( \alpha+\frac{1}{K} \right)+\mathcal O \left( 2^{-\kappa} \right)
\end{align*}
where we define $\Phi_0:=J^{*} - J \left(\theta_0\right)+\frac{2\eta}{p} \left\| \frac{1}{K}\sum_{k \in [K]}\hat{v}_0^{(k)}-\nabla J \left( \theta_0 \right) \right\|^2$ with $J^{*}:=\max_{\theta \in \R^d} J(\theta)$.
\end{theorem}

\begin{corollary}\label{cor:DecByzPG}
In the setting of Theorem \ref{thm:DecByzPG}, by choosing $p=1/N$ and $\kappa=\Theta \left( \max \left\{ \log \left( NK \right),\log \left( \epsilon^{-1} \right) \right\} \right)$, the expected number of trajectories that need to be sampled per agent, to achieve a $K$-agent $\bar\alpha$-tolerant $\epsilon$-approximate solution as in Definition \ref{def:solution}, is
\begin{align*}
\mathcal O \left( \frac{\alpha^{3/2}}{\epsilon^4}+\frac{\alpha^{1/2}}{K\epsilon^4}+\frac{\alpha^{1/2}}{K^{1/2}\epsilon^3}+\frac{1}{K\epsilon^3} \right).
\end{align*}
\end{corollary}

In particular, if $\alpha=0$, we need $\mathcal O (K^{-1}\epsilon^{-3})$ trajectories in expectation which matches with our respective result from Corollary \ref{cor:ByzPG}. The same sample complexity has been obtained in \cite{jiang2022mdpgt} for a momentum-based decentralized PG method that, however, lacks fault-tolerance. For constant $\alpha > 0$, we need $\mathcal O(\epsilon^{-4})$ trajectories in expectation which in our setting matches for example the complexity of single-agent vanilla PG \cite{papini2021safe}. If, e.g., a constant number of agents are Byzantine, i.e., $\alpha=\Theta(K^{-1})$, we get a complexity of $\mathcal O(K^{-3/2}\epsilon^{-4}+K^{-1}\epsilon^{-3})$. Hence asymptotic speed-up w.r.t.\ the number of agents is possible despite the presence of Byzantine agents.

\begin{remark}
Besides sample complexity, we prefer algorithms with low communication complexity. Due to \textbf{Avg-Agree}$_{\kappa}$, each of the $T$ iterations of \textsc{DecByzPG} involves $\kappa=\Theta( \max \{ \log (NK),\log (\epsilon^{-1})\})$ rounds of all-to-all communication, each consisting of $\mathcal O(K^2)$ messages containing a vector in $\R^d$. We point out that the logarithmic number of rounds is crucial for the practicality of our decentralized algorithm, as otherwise the cost of communication may outweigh the benefits of the lower sample complexity gained from collaboration.
\end{remark}

Due to space constraints, the full proofs of Theorem \ref{thm:DecByzPG} and Corollary \ref{cor:DecByzPG}, as well as all required lemmas are deferred to Appendix \ref{app:DecByzPG-proofs}. Here, we want to focus on one key argument of the proof responsible for controlling the diameter of agents' local parameters. Before stating and proving the two respective lemmas, we introduce additional notation: Recall that $\mathcal H_t \subset [K]$ is the set of \emph{honest}, i.e., non-Byzantine agents as in Assumption \ref{ass:byz-fraction} with $|\mathcal H_t|\geq (1-\alpha) K$. In addition, with $\bar\alpha=\alpha+\bar\epsilon < \alpha_{\max} = 1/4$, denote the diameter of a tuple of vectors by $\Delta_2 \left( \cdot \right)^2$, e.g., for some $S \subseteq [K]$, let
\begin{align*}
\Delta_2 \left( \langle \theta_t^{(i)} \rangle_{i \in S} \right) := \max_{i,j \in S} \| \theta_t^{(i)}-\theta_t^{(j)} \|
\end{align*}
and consider the set
\begin{align*}
\mathcal G_t := \argmin_{S \subset \mathcal H_t, |S| \geq (1-\bar\alpha)K} \Delta_2 \left( \langle \tilde{\theta}_t^{(i)} \rangle_{i \in S} \right) \subset \mathcal H_t
\end{align*}
which we will call the set of \emph{good} agents. As we will show below, the diameter of good agents' parameters exhibits good concentration in the sense that we obtain stronger bounds as would hold for the expected diameter of all honest agents' parameters. The diameter of good agents' parameters after agreement will frequently occur as an error term which we denote by
\begin{align}
\label{eqn:diam-bound}
\mathcal E^{\Delta}_t:=\Delta_2 \left( \langle \theta_t^{(i)} \rangle_{i \in \mathcal G_t} \right)^2.
\end{align}
Finally, we abbreviate
\begin{align*}
\widetilde{\mathcal T}_{1,t} &:= \frac{1}{|\mathcal G_t|(|\mathcal G_t|-1)} \sum_{i,l \in \mathcal G_t} \E \left[ \| \tilde{v}_t^{(i)}-\tilde{v}_t^{(l)} \|^2 \mid c_t=1 \right], \\
\widetilde{\mathcal T}_{0,t} &:= \frac{1}{|\mathcal G_t|(|\mathcal G_t|-1)} \sum_{i,l \in \mathcal G_t} \E \left[ \| \tilde{v}_t^{(i)}-\tilde{v}_t^{(l)} \|^2 \mid c_t=0 \right].
\end{align*}
The following lemma bounds the diameter of good agents' parameters after aggregation and before agreement in iteration $t$, distinguishing between the two cases given by the probabilistic switch.
\begin{lemma}\label{lem:diam}
For any $\bar{\epsilon}>0$, it holds that
\begin{align*}
\E \left[ \Delta_2 \left( \langle \tilde{\theta}_{t+1}^{(i)} \rangle_{i \in \mathcal G_t} \right)^2 \mid c_t=1 \right] &\leq 2 \E[\mathcal E^\Delta_t] + \frac{10 \eta^2C_{ra}\alpha\widetilde{\mathcal T}_{1,t}}{\bar{\epsilon}},\\
\E \left[ \Delta_2 \left( \langle \tilde{\theta}_{t+1}^{(i)} \rangle_{i \in \mathcal G_t} \right)^2 \mid c_t=0 \right] &\leq 2 \E[\mathcal E^\Delta_t] + \frac{10 \eta^2C_{ra}\alpha\widetilde{\mathcal T}_{0,t}}{\bar{\epsilon}}.
\end{align*}
\end{lemma}
\begin{proof}
First, we define
\begin{align*}
S_t := \argmin_{S \subset \mathcal G_t, |S| \geq (1-(\alpha+\bar{\epsilon}))K} \Delta_2 \left( \langle v^{(i)}_t \rangle_{i \in S_t} \right) \subset \mathcal H_t \subset [K]
\end{align*}
and aim to bound $\E\left[ \Delta_2(\langle v^{(i)}_t \rangle_{i \in S_t})^2 \right]$. Observe that we have
\begin{align*}
\Delta_2 \left( \langle v^{(i)}_t \rangle_{i \in S_t} \right)^2 \leq \max_{i \in S_t} \| \bar{\tilde{v}}_t-v^{(i)}_t \|^2.
\end{align*}
For $i \in \mathcal H_t$, let $\overline{\mathcal T}_i$ be such that $\E[ \| \bar{\tilde{v}}_t-v^{(i)}_t \|^2] \leq \overline{\mathcal T}_i$ (we will plug in the right bound $\overline{\mathcal T}_i$ later) where $\bar{\tilde{v}}_t:= \frac{1}{|\mathcal G_t|} \sum_{i \in \mathcal G_t} \tilde{v}^{(i)}_t$, and let $X_i$ be indicator random variables for the events
\begin{align*}
\| \bar{\tilde{v}}_t-v^{(i)}_t \|^2 \geq \frac{2}{\bar\epsilon} \cdot \overline{\mathcal T}_i.
\end{align*}
Let $X=\sum_{i \in \mathcal H_t} X_i$. Our goal is to upper bound $\E X_i=\Pr \left[ X_i=1 \right]$ in order to use Chernoff concentration bounds on $X$.

By Lemma \ref{lem:markov} (see Appendix \ref{app:useful}), we get
\begin{align*}
\pr \left[ X_i=1 \right] &= \pr \left[ \| \bar{\tilde{v}}_t-v^{(i)}_t \|^2 \geq \frac{2}{\bar \epsilon} \cdot \overline{\mathcal T}_i \right] \\
&\leq \pr \left[ \| \bar{\tilde{v}}_t-v^{(i)}_t \|^2 \geq \frac{2}{\bar \epsilon} \cdot \E \left[ \| \bar{\tilde{v}}_t-v^{(i)}_t \|^2 \right] \right] \\
&= \pr \left[ \| \bar{\tilde{v}}_t-v^{(i)}_t \| \geq \sqrt{\frac{2}{\bar \epsilon}} \cdot \sqrt{\E \left[ \| \bar{\tilde{v}}_t-v^{(i)}_t \|^2 \right]} \right] \\
&\leq \frac{\E \left[ \| \bar{\tilde{v}}_t-v^{(i)}_t \|^2 \right]}{\frac{2}{\bar \epsilon} \E \left[ \| \bar{\tilde{v}}_t-v^{(i)}_t \|^2 \right]} \\
&= \frac{\bar \epsilon}{2}.
\end{align*}
Let $E$ be the ``bad case'', i.e., the event that we have $X \geq (\bar{\epsilon}+\alpha)K$. By Lemma \ref{lem:chernoff} (see Appendix \ref{app:useful}), with $\delta=1+\frac{2\alpha}{\bar \epsilon}$ and $\hat{p}$ as bounded above, we get
\begin{align*}
\pr[E] = \pr \left[ X \geq (\bar{\epsilon}+\alpha)K \right] &= \pr \left[ X \geq \left( 1+\delta \right) \frac{\bar{\epsilon}K}{2} \right] \\
&\leq \exp \left( -\frac{\delta^2 \bar \epsilon K}{4+2\delta} \right) \\
&= \exp \left( -\frac{K(2\alpha+\bar\epsilon)^2}{4\alpha+6\bar\epsilon} \right) \\
&\leq \exp \left( -\frac{\bar\epsilon K}{6} \right).
\end{align*}
With $\overline{\mathcal T}:=\max_{i \in \mathcal H_t} \overline{\mathcal T}_i$, by the law of total expectation, we then have
\begin{align*}
\E \left[ \Delta_2 \left( \langle v^{(i)}_t \rangle_{i \in S_t} \right)^2 \right] &\leq \E \left[ \Delta_2 \left( \langle v^{(i)}_t \rangle_{i \in S_t} \right)^2 \mid \bar{E} \right] \cdot \underbrace{\pr \left[ \bar{E} \right]}_{\leq 1} \\
&\quad\quad+ \E \left[ \Delta_2 \left( \langle v^{(i)}_t \rangle_{i \in S_t} \right)^2 \mid E \right] \cdot \pr \left[ E \right] \\
&\leq \frac{2}{\bar{\epsilon}} \overline{\mathcal T} + K \overline{\mathcal T} \cdot \exp \left( -\frac{\bar{\epsilon}K}{6} \right) \\
&\leq \frac{5}{\bar{\epsilon}} \overline{\mathcal T}
\end{align*}
where in the first step, for the expectation conditioned on $E$ we union-bound the $\max$ by introducing a factor $K$, and in the second step we use the fact that the function $f(x)=xe^{-\beta x}$ has a global maximum with value $\frac{1}{\beta e}$.

Remains to use this bound on $\E[ \Delta_2(\langle v^{(i)}_t \rangle_{i \in S_t})^2]$ in order to obtain the desired bound on $\E[ \Delta_2(\langle \tilde{\theta}^{(i)}_{t+1} \rangle_{i \in \mathcal G_t})^2]$ which follows straightforwardly since
\begin{align*}
&\E \left[ \Delta_2 \left( \langle \tilde{\theta}^{(i)}_{t+1} \rangle_{i \in \mathcal G_t} \right)^2 \right] \\
= \;&\E \left[ \max_{i,l \in \mathcal G_t} \| \tilde{\theta}_{t+1}^{(i)}-\tilde{\theta}_{t+1}^{(l)} \|^2 \right]\\
\leq \;&2 \E \left[ \max_{i,l \in \mathcal G_t} \| \theta_{t}^{(i)}-\theta_{t}^{(l)} \|^2 \right]+2\eta^2 \E \left[ \max_{i,l \in \mathcal G_t} \| v_t^{(i)}-v_t^{(l)}\|^2 \right] \\
\leq \;&2 \E[\mathcal E^\Delta_t] + 2 \eta^2 \E \left[ \Delta_2 \left( \langle v^{(i)}_t \rangle_{i \in S_t} \right)^2 \right] \\
\leq \;&2 \E[\mathcal E^\Delta_t] + \frac{10\eta^2}{\bar{\epsilon}}  \overline{\mathcal T}.
\end{align*}
What can we plug in for $\overline{\mathcal T}$? Observe that $v_t^{(i)}$ is the result of aggregation of inputs with average $\bar{\tilde{v}}_t:= \frac{1}{|\mathcal G_t|} \sum_{i \in \mathcal G_t} \tilde{v}^{(i)}_t$. Therefore, by Definition \ref{def:aggregate}, for any $i \in \mathcal G_t$, we have
\begin{align*}
\E \left[ \| v_t^{(i)}-\bar{\tilde{v}}_t \|^2 \right]
&\leq \frac{C_{ra} \alpha}{|\mathcal G_t|(|\mathcal G_t|-1)} \sum_{i,l \in \mathcal G_t} \E \left[ \| \tilde{v}_t^{(i)}-\tilde{v}_t^{(l)} \|^2 \right].
\end{align*}
Thus, we can distinguish between conditioning our expectation on $c_t=0$ and $c_t=1$ and using the respective bounds $C_{ra}\alpha \widetilde{\mathcal T}_{0,t}$ and $C_{ra}\alpha \widetilde{\mathcal T}_{1,t}$, the result follows.
\end{proof}

With this bound on the diameter of intermediate local parameters $\tilde\theta_t^{(i)}$ at hand, we can now derive the desired bound on parameters after agreement that only depends on the averaging agreement parameter $\kappa$.
\begin{lemma}\label{lem:diam-error}
For all iterations $t \leq T$, there exists $\overline{\mathcal E}^{\Delta,\kappa}$ such that $\E [\mathcal E^{\Delta}_t] \leq \overline{\mathcal E}^{\Delta,\kappa}$ and
\begin{align*}
\overline{\mathcal E}^{\Delta,\kappa} \leq \mathcal O \left( 2^{-\kappa} \right).
\end{align*}
\end{lemma}
\begin{proof}
We proceed by induction on $t$. For $t=0$, $\mathcal E_t^{\Delta}=0$ due to the common initialization $\theta^{(k)}=\theta_0$ for all $k \in [K]$. Suppose for some $t<T-1$, $\E[\mathcal E_t^{\Delta}] \leq \overline{\mathcal E}^{\Delta,\kappa} \leq \mathcal O \left( 2^{-\kappa} \right)$. In iteration $t$, applying Lemma \ref{lem:diam} and the bounds on $\widetilde{\mathcal T}_{0,t}$ and $\widetilde{\mathcal T}_{0,t}$ from Lemma \ref{lem:DecByzPG-variance-bound} (see Appendix \ref{app:DecByzPG-proofs}), one can observe that in both cases $c_{t-1}=1$ and $c_{t-1}=0$, the expected diameter of $\tilde{\theta}_t^{(i)}$'s for good agents $i \in \mathcal G_t$ is (loosely) bounded by $\mathcal O(1)$. Hence by the definition of averaging agreement, see Definition \ref{def:agree}, we have $\E \left[ \mathcal E_t^{\Delta} \right] \leq \overline{\mathcal E}^{\Delta,\kappa}$ with
\begin{align*}
\overline{\mathcal E}^{\Delta,\kappa} \leq \frac{\E \left[ \Delta_2 \left( \langle \tilde{\theta}_t^{(i)} \rangle_{i \in \mathcal G_t} \right) \right]}{2^{\kappa}} \leq \mathcal O \left( 2^{-\kappa} \right).
\end{align*}
\end{proof}

\section{Experiments}
\label{sec:exp}
In order to corroborate our theoretical findings, we empirically study the performance of the proposed methods w.r.t.\ the properties suggested by Corollary \ref{cor:DecByzPG}, i.e., (a) speed-up when increasing the number of agents $K$, and (b) resilience against various Byzantine attacks. We focus on our main contribution regarding the more challenging decentralized setting here (i.e.\ \textsc{DecByzPG}), and defer experiments for \textsc{ByzPG} to Appendix \ref{app:add-experiments}.

\textbf{Environments and Setup.} We consider two common RL benchmarks, CartPole \cite{barto1983neuronlike} and LunarLander. For all experiments, we report average returns of honest agents (y-axis) in terms of the trajectories that have been sampled per agent (x-axis). To visualize potential variance in our experiments, all plots show the respective mean and standard deviation across 10 independent runs. Further details, including hyperparameters, can be found in Appendix \ref{app:experiment-setup}.

\subsection{\textsc{DecByzPG} without Byzantine Agents}

\begin{figure}[h]
\centering
\includegraphics[width=\linewidth]{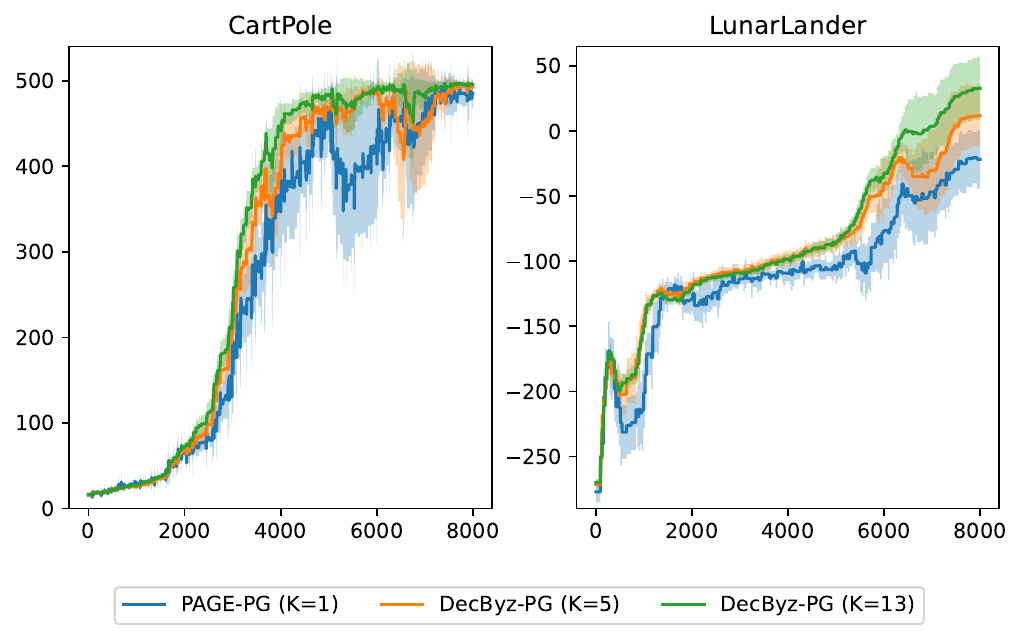}
\caption{Performance of \textsc{DecByzPG} for different federation sizes when all agents behave honestly (i.e.\ $\alpha=0$).}
\label{fig:speedup}
\Description{Performance of \textsc{DecByzPG} for different federation sizes when all agents behave honestly (i.e.\ $\alpha=0$).}
\end{figure}
In Figure \ref{fig:speedup}, we consider \textsc{DecByzPG} in the case $\alpha=0$, with $K=1$ (which is equivalent to \textsc{PAGE-PG} \cite{gargiani2022page}), $K=5$, and $K=13$. Speed-up with increasing number of agents is observable in both environments, as suggested by Corollary \ref{cor:DecByzPG}. Such faster convergence provides empirical evidence motivating agents to join a decentralized federation.

\subsection{\textsc{DecByzPG} under Attack}
\begin{figure*}
\centering
\includegraphics[width=0.9\linewidth]{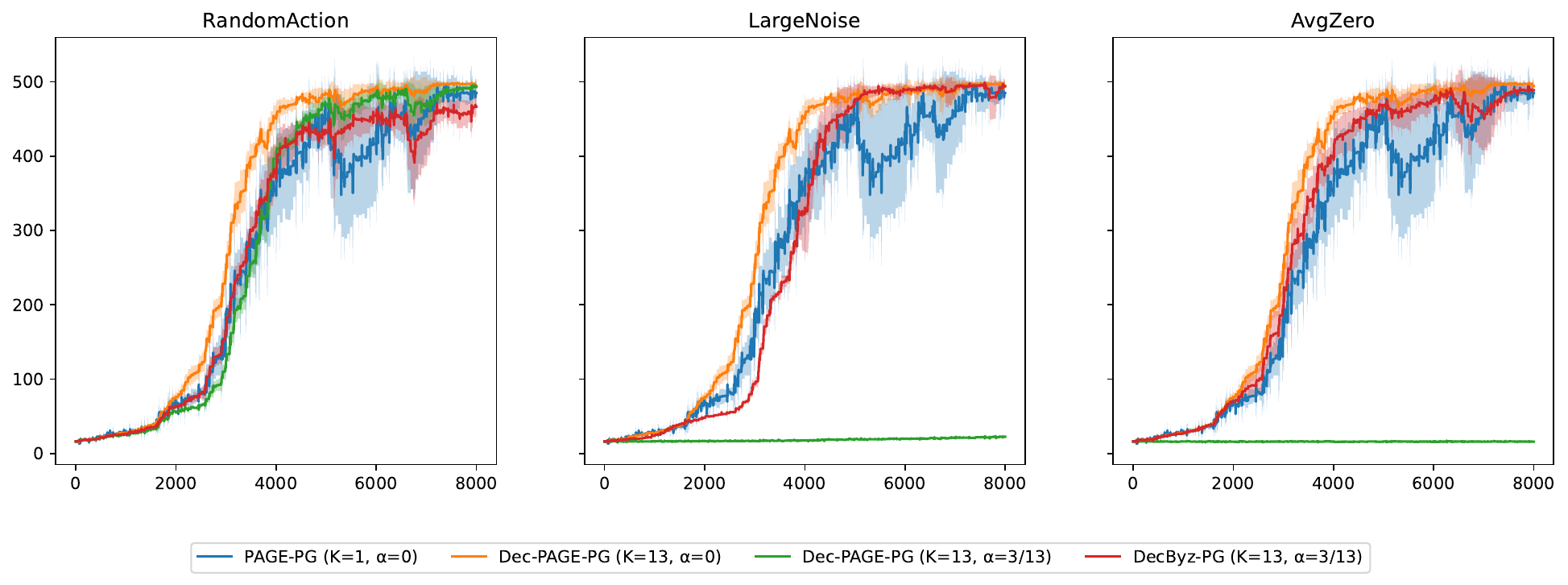}
\caption{Performance \& resilience of \textsc{DecByzPG} for CartPole w.r.t.\ our three attack types.}
\label{fig:cartpole-attack}
\Description{Performance \& resilience of \textsc{DecByzPG} for CartPole w.r.t.\ our three attack types.}
\end{figure*}

\begin{figure*}
\centering
\includegraphics[width=0.9\linewidth]{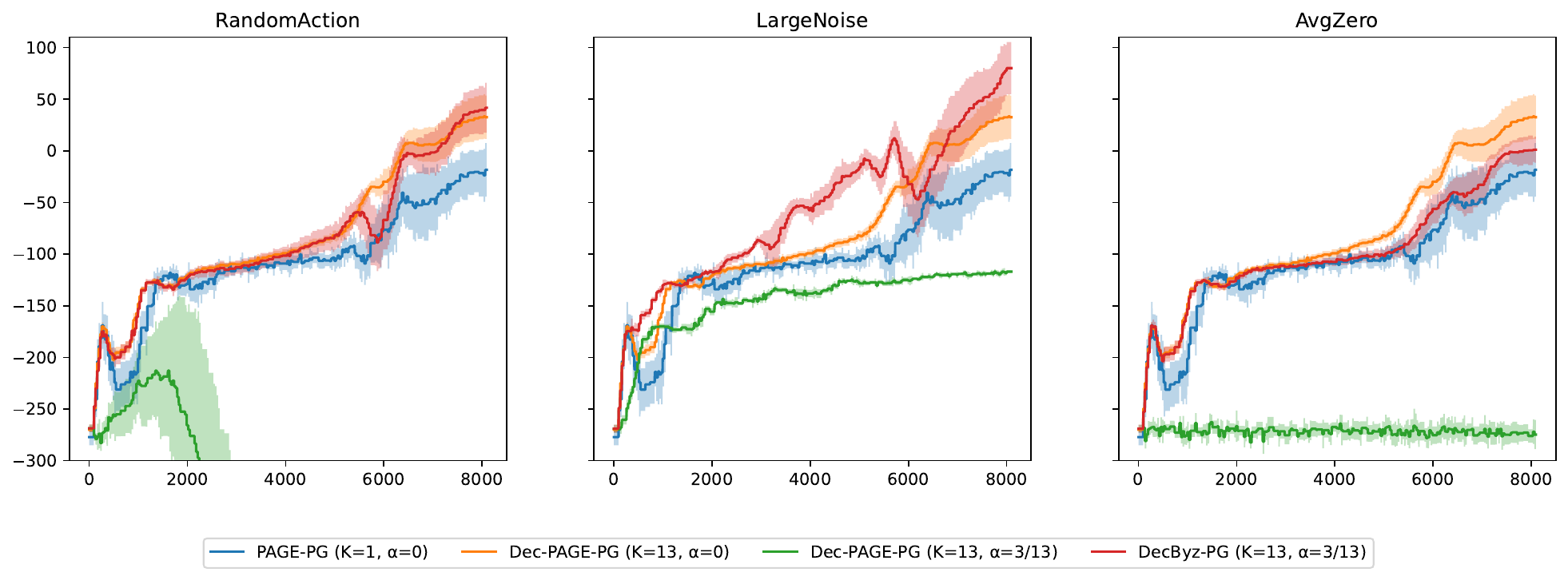}
\caption{Performance \& resilience of \textsc{DecByzPG} for LunarLander w.r.t.\ our three attack types.}
\label{fig:lunar-attack}
\Description{Performance \& resilience of \textsc{DecByzPG} for LunarLander w.r.t.\ our three attack types.}
\end{figure*}

\textbf{Choice of attacks.} In previous work \cite{fan2021fault}, Byzantine attacks are constructed by making random modifications to an agent's interaction with its environment, by e.g.\ choosing an action u.a.r.\ instead of following the current policy (here denoted \textbf{RandomAction}), adding noise to the reward, or randomly flipping the reward's sign. We find that in our setting, for simple environments such as CartPole, robustness to such attacks is often already given for naively collaborating agents. This behavior is exemplified by our experiments under the \textbf{RandomAction} attack. Thus, even though \textsc{DecByzPG} is also resilient to such attacks, a stronger adversary is needed to demonstrate \textsc{DecByzPG}'s advantage over naive methods. \textbf{LargeNoise} lets Byzantine agents directly send noise instead of gradients obtained from noisy interactions. Even though introducing noise may generally also have beneficial effects on convergence speed (e.g.\ due to improved exploration), by choosing the variance large enough, such benefits are outweighed. The third attack, \textbf{AvgZero} leverages the power of Byzantine knowledge and collaboration. Gradients sent by Byzantines are chosen such that when averaged with gradients sent by honest agents, the result will be close to zero.

In Figure \ref{fig:cartpole-attack} and \ref{fig:lunar-attack}, we compare \textsc{DecByzPG} under above attacks to (a) \textsc{PAGE-PG} \cite{gargiani2022page}, the SOTA single-agent PG method that \textsc{DecByzPG} reduces to when $K=1$, and (b) \textsc{Dec-PAGE-PG}, a naive decentralized (but not fault-tolerant) version of \textsc{PAGE-PG} where aggregation of gradients is done by averaging, and no agreement mechanism is used. Note that for experiments involving Byzantine agents, we choose their quantity to be the largest for which Assumption \ref{ass:byz-fraction}, and hence the guarantees of Theorem \ref{thm:DecByzPG}, still hold (i.e. 3 out of 13 agents are Byzantine).

For both environments and all attacks, we can observe that \textsc{DecByzPG} performs nearly on par with the unattacked \textsc{Dec-PAGE-PG}. This empirically supports the Byzantine fault-tolerance of \textsc{DecByzPG}. Furthermore, for CartPole, as expected, \textbf{LargeNoise} and \textbf{AvgZero} are highly effective against the non-fault-tolerant method, while as previously remarked, \textbf{RandomAction} barely shows any effect. For the more difficult task of LunarLander, already \textbf{RandomAction} breaks \textsc{Dec-PAGE-PG}. Lastly, we point out that in all cases \textsc{DecByzPG} with $K=13$ and $\alpha > 0$ outperforms \textsc{PAGE-PG} with $K=1$ (and $\alpha=0$), meaning that in our experiments, despite the presence of Byzantines, joining the federation is empirical beneficial for faster convergence.

\section{Conclusion}
\label{sec:conclusion}
We described and analyzed a federated decentralized Byzantine fault-tolerant PG algorithm. As a warm-up, we combined variance-reduced PG methods with results from Byzantine-tolerant non-convex optimization to obtain a new centralized algorithm under standard assumptions. We then use ideas from Byzantine robust aggregation and agreement to generalize our approach to the significantly more challenging decentralized setting. As a result, we obtained the first sample complexity guarantees for Byzantine fault-tolerant decentralized federated non-convex optimization. We thus believe that our technical contributions are more generally applicable and may therefore open up directions for future research. Moreover, the provided empirical results for standard RL benchmark tasks support our theory and promise practical relevance of our method.

\newpage

\bibliography{references}
\bibliographystyle{unsrtnat}

\newpage
\appendix
\onecolumn

\section*{\huge{Appendix}}

\section{Background in RL \& Details on Robust Aggregation and Agreement}

\subsection{Details on Gradient Estimation and Importance Sampling in PG Methods}
\label{app:sampling}
Throughout the paper, we have used $g\left( \tau \mid \theta \right)$ to denote an unbiased estimator of $\nabla J (\theta)$. It is known, see e.g.\ \cite{fan2021fault}, that in our setting of episodic MDP with trajectory horizon $H$, $g\left( \tau \mid \theta \right)$ can for example be implemented as REINFORCE \cite{williams1992REINFORCE},
\begin{align*}
g\left(\tau \mid \theta\right)=\left(\sum_{h=0}^{H-1} \nabla_{\theta} \log \pi_{\theta}\left(a_h \mid s_h\right)\right)\left(\sum_{h=0}^{H-1} \gamma^h \mathcal{R}\left(s_h, a_h\right)-C_b\right),
\end{align*}
or GPOMDP \cite{baxter2001GPOMDP},
\begin{align*}
g\left(\tau \mid \theta\right)=\sum_{h=0}^{H-1}\left(\sum_{t=0}^h \nabla_{\theta} \log \pi_{\theta}\left(a_t \mid s_t\right)\right)\left(\gamma^h r\left(s_h, a_h\right)-C_{b_h}\right).
\end{align*}
with $C_b$ and $C_{b_h}$ denoting the corresponding baselines. Under Assumption \ref{ass:bounded-policy}, Proposition \ref{prop:smoothness} is known to hold for both REINFORCE and GPOMDP \cite{xu2020improved}. For the experiments, our implementation will be based on GPOMDP, as it has generally been reported to have lower variance and thus yields better performance than REINFORCE \cite{papini2018stochastic,xu2020improved}.

Variance-reduced stochastic gradient descent methods such as SVRG \cite{johnson2013accelerating} and PAGE \cite{li2021page} rely on the ability to sample gradients at points that are not the current iterate. In PG however, the underlying sample distribution depends on the current parameters $\theta_t$. To overcome this challenge of non-stationarity, we employ the commonly used importance sampling technique and follow \cite{gargiani2022page} in defining, for any $\theta_t,\theta_{t-1} \in \R^d$, and any trajectory $\tau$,
\begin{align*}
g^{\omega_{\theta_t}} \left( \tau \mid \theta_{t-1} \right) := \omega \left( \tau \mid \theta_t,\theta_{t-1} \right) g \left( \tau \mid \theta_{t-1} \right)
\end{align*}
with importance weight $\omega \left( \tau \mid \theta_t,\theta_{t-1} \right):=\frac{p \left( \tau \mid \theta_{t-1} \right)}{p \left( \tau \mid \theta_t \right)}$. It can be shown \cite{papini2018stochastic} that this yields an unbiased estimate of $\nabla J \left( \theta_{t-1} \right)$ despite trajectory $\tau$ being sampled from the current policy based on $\theta_t$, i.e.,
\begin{align*}
\E_{\tau \sim p \left(\cdot\mid\theta_t\right)} \left[ g^{\omega_{\theta_t}} \left( \tau \mid \theta_{t-1} \right) \right] = \nabla J \left( \theta_{t-1} \right).
\end{align*}
At this point, we also introduce the following additional notation that will be used throughout the proofs of Theorem \ref{thm:ByzPG} and \ref{thm:DecByzPG}: For any $\theta_t, \theta_{t-1} \in \R^d$, and $M \in \mathbb N$, let
\begin{align*}
\hat{\Delta}^M \left( \theta_t,\theta_{t-1} \right) &:= \frac{1}{M} \sum_{j=1}^M g\left(\tau_{t,j} \mid \theta_t\right)-\frac{1}{M} \sum_{j=1}^M g^{\omega_{\theta_t}}\left( \tau_{t,j} \mid \theta_{t-1} \right)
\end{align*}
and
\begin{align*}
\Delta \left( \theta_t,\theta_{t-1} \right) &:= \nabla J \left( \theta_t \right) - \nabla J \left( \theta_{t-1} \right)
\end{align*}
where $\tau_{t,j} \sim p \left( \cdot \mid \theta_t \right)$ for $j=1,\dots,M$.

\subsection{Implementation of Robust Aggregation}
\label{app:impl-aggregation}

The notion of $(\alpha,C_{ra})$-robust aggregation, see Definition \ref{def:aggregate}, is adopted from \cite{gorbunov2023variance}, and has originally appeared in a similar form in \cite{karimireddy2020byzantine}. Known implementations satisfying our requirement of $C_{ra}=\mathcal O \left( 1 \right)$ include \emph{Krum} \cite{blanchard2017machine} and \emph{Robust Federated Averaging (RFA)} \cite{pillutla2022robust}, where both must be used in conjunction with \emph{bucketing} \cite{karimireddy2020byzantine}.

\textbf{Krum.} Denoting by $S_i$ the $\lceil (1-\alpha)K \rceil$ closest neighbors to $\theta^{(i)}$ among $\theta^{(1)},\dots,\theta^{(K)}$ in Euclidean norm, we let
\begin{align*}
\texttt{Krum} \left( \theta^{(1)},\dots,\theta^{(K)} \right) := \argmin_{\theta^{(i)} \,\text{s.t.}\, i \in [K]} \sum_{j \in S_i} \left\| \theta^{(j)}-\theta^{(i)} \right\|^2.
\end{align*}
Note that due to the computation of pairwise distances, Krum has runtime complexity $\mathcal O \left( K^2 \right)$.
\textbf{RFA.} We define
\begin{align*}
\texttt{RFA} \left( \theta^{(1)},\dots,\theta^{(K)} \right) := \argmin_{\theta \in \R^d} \sum_{i \in [K]} \left\| \theta-\theta^{(i)} \right\|^2
\end{align*}
which corresponds to finding the geometric median---a problem that does not have a closed form solution. However, efficient iterative approximation methods, such as the smoothed Weiszfeld algorithm \cite{weiszfeld1937point,pillutla2022robust}, exist.
\textbf{Bucketing.} Instead of directly aggregating the inputs vectors, \cite{karimireddy2020byzantine} proposes to apply existing aggregators to the means of buckets of size $\lfloor \alpha_{\max}/\alpha \rfloor$, where inputs are randomly assigned to buckets (see Algorithm 1 in \cite{karimireddy2020byzantine} for the detailed procedure). It has been shown that this can turn aggregation methods not satisfying Definition \ref{def:aggregate} into robust aggregators.

For completeness, we restate the relevant parts summarized by Theorem D.1 of \cite{gorbunov2023variance} in the following lemma.

\begin{lemma}[Implementations of $(\alpha,C_{ra})$-robust aggregation]
For $0<\alpha<\alpha_{\max}$, $C_{ra}=\mathcal O \left( 1 \right)$, and bucket size $\lfloor \alpha_{\max}/\alpha \rfloor$, it holds that
\begin{itemize}
\item bucketing with Krum is an $\left( \alpha,C_{ra} \right)$-robust aggregator for $\alpha_{\max}=1/4$, and
\item bucketing with RFA is an  $\left( \alpha,C_{ra} \right)$-robust aggregator for $\alpha_{\max}=1/2$.
\end{itemize}
\end{lemma}
Note that this means, in order to achieve the values $\alpha_{\max}$ stated in Assumption \ref{ass:byz-fraction}, i.e., $1/2$ in the centralized, and $1/4$ in the decentralized case, we may use RFA for \textsc{ByzPG}, and either RFA or Krum for \textsc{DecByzPG}.

\subsection{Implementation of Averaging Agreement}
\label{app:impl-agreement}

Averaging agreement, as in Definition \ref{def:agree}, has been introduced by \cite{el2021collaborative}, together with two possible implementations. One of them, \emph{Minimum Diameter Averaging (MDA)} also satisfies our stronger requirement of $C_{avg}=\mathcal O \left( 1 \right)$.

\textbf{MDA.} For each $k \in [K]$, let
\begin{align*}
\textsc{Byz}^{(k)} \left( \langle \theta^{(k')} \rangle_{k' \in [K]} \right)=\langle \textsc{Byz}^{(k,i)} \left( \langle \theta^{(k')} \rangle_{k' \in [K]} \right) \rangle_{i \in [K]}
\end{align*}
be the set of vectors received by agent $k$ (from agent $i$) after one round of all-to-all parameter broadcast of $\langle \theta^{(k')} \rangle_{k' \in [K]}$, subject to some Byzantine attack, under Assumption \ref{ass:byz-fraction}. Furthermore, let
\begin{align*}
\texttt{MDA} \left( \langle \theta^{(k)} \rangle_{k \in [K]} \right) := \langle \theta^{(k)} \rangle_{k \in S^{*}} \quad\text{where}\quad S^{*} = \argmin_{S \subset [K], |S| \geq (1-\bar\alpha)K} \Delta_2 \left( \langle \theta^{(k)} \rangle_{k \in S} \right).
\end{align*}
The MDA averaging agreement mechanism is then defined by running Algorithm \ref{alg:mda} at all agents $k \in [K]$ concurrently, for $\kappa$ iterations.

\begin{algorithm}[ht]
\caption{\textsc{MDA} at $k$-th agent}\label{alg:mda}
\begin{algorithmic}[1]
\State {\bfseries input:} $\theta^{(k)}$
\For{$\kappa$ iterations}
\State broadcast $\theta^{(k)}$ to all other agents and receive $\mathcal B:=\textsc{Byz}^{(k)} \left( \langle \theta^{(k')} \rangle_{k' \in [K]} \right)$
\State let $\mathcal M \leftarrow \texttt{MDA} \left( \mathcal B \right)$ \label{line:mda}
\State $\theta^{(k)} \leftarrow \frac{1}{|\mathcal M|} \sum_{\theta \in \mathcal M} \theta$
\EndFor
\State {\bfseries output: $\theta^{(k)}$}
\end{algorithmic}
\end{algorithm}

The following is a special case of \cite{el2021collaborative}'s Theorem 4 for our synchronous setting, restated here for convenience.
\begin{lemma}
\label{lem:mda-avg-agree}
Under Assumption \ref{ass:byz-fraction}, i.e., with $\alpha_{\max}=1/4$, and assuming synchronous communication, MDA as in Algorithm \ref{alg:mda} achieves $C_{avg}$-averaging agreement for $C_{avg}=\mathcal O(1)$.
\end{lemma}
However, it needs to be pointed out that for any constant $\bar\alpha>0$, finding the subset that minimizes the diameter in Line \ref{line:mda} of Algorithm \ref{alg:mda} has computational complexity exponential in $K$, and for larger $K$ is therefore not suitable in practice. As proposed by \cite{el2021collaborative}, a computationally efficient alternative exists, which we call \emph{Greedy Diameter Averaging (GDA)}. We define
\begin{align*}
\texttt{GDA}_k \left( \langle \theta^{(k)} \rangle_{k \in [K]} \right) := \langle \theta^{(k)} \rangle_{k \in S^{*}} \quad\text{where}\quad S^{*} = \argmin_{S \subset [K], |S| \geq (1-\bar\alpha)K} \sum_{i \in S} \left\| \theta^{(i)}-\theta^{(k)} \right\|^2.
\end{align*}
Note that unlike for MDA, in the case of GDA, the set $S^{*}$ can be found in time $\mathcal O(K)$ by choosing the $\lceil (1-\bar\alpha)K \rceil$ parameter vectors closest to $\theta^{(k)}$. As mentioned in \cite{el2021collaborative}, replacing \texttt{MDA} with \texttt{GDA}$_k$ in Algorithm \ref{alg:mda} still achieves $C_{avg}$-averaging agreement with $C_{avg}=\mathcal O(1)$ but comes at a slight cost in the fraction of tolerable Byzantines, namely requiring $\alpha_{\max}=1/5$ instead of $1/4$ in Assumption \ref{ass:byz-fraction}.

\subsection{Useful Facts}
\label{app:useful}

In this section, we collect some simple relations and lemmas that will be helpful throughout our proofs. First, we recall some basic facts.
\begin{lemma}[Basic facts]\label{lem:basics}
For any $x,y,a_1,\dots,a_n \in \R^d$, $p \in (0,1]$, and $\beta > 0$, it holds that
\begin{align}
\langle x,y \rangle &= \frac{\left\| x \right\|^2}{2} + \frac{\left\| x \right\|^2}{2} - \frac{\left\| y-x \right\|^2}{2} \label{eqn:scalar-prod} \\
\left\| \sum_{i=1}^n a_i \right\|^2 &\leq n \sum_{i=1}^n \left\| a_i \right\|^2 \label{eqn:norm-of-sum} \\
\left\| x+y \right\|^2 &\leq \left( 1+\beta \right) \left\| x \right\|^2 + \left( 1+\beta^{-1} \right) \left\| y \right\|^2 \label{eqn:beta-norm-sum} \\
\left( 1-p \right) \left( 1+\frac{p}{2} \right) &\leq 1-\frac{p}{2} \label{eqn:p-prod-1} \\
1.1p \left( 1-p \right) + \left( 1-p \right)^2 \left( 1+\frac{p}{4} \right) &\leq 1 - \frac{p}{2} \label{eqn:p-prod-2}
\end{align}
\end{lemma}

Next, we show a simple bound on the distance to a mean vector in terms of average pairwise distances.
\begin{lemma}\label{lem:avg-vec-bound}
Let $\theta^{(1)},\dots,\theta^{(K)} \in \R^d$, and $\bar{\theta}=\frac{1}{K} \sum_{j \in [K]} \theta^{(j)}$. Then, for any $i \in [K]$,
\begin{align*}
\left\| \bar{\theta}-\theta^{(i)} \right\|^2 \leq \frac{1}{K} \sum_{j \in [K]} \left\| \theta^{(j)}-\theta^{(i)} \right\|^2
\end{align*}
\end{lemma}
\begin{proof}
Observe that,
\begin{align*}
\left\| \bar{\theta}-\theta^{(i)} \right\|^2 &= \Bigg\| \frac{1}{K} \sum_{j \in [K]} \theta^{(j)}-\theta^{(i)} \Bigg\|^2 = \frac{1}{K^2} \Bigg\| \sum_{j \in [K]} \theta^{(j)}-\theta^{(i)} \Bigg\|^2
\end{align*}
from which the result follows using Lemma \ref{lem:basics}, (\ref{eqn:norm-of-sum}).
\end{proof}

The following lemma shows how repeated, independent sampling reduces the variance of gradient estimates.
\begin{lemma}\label{lem:indep-sampling-variance}
Let Assumption \ref{ass:fin-var} hold. Then, for any $M \in \N$, $\theta_t^{(i)} \in \R^d$, and $\tau_{t,j}^{(i)} \sim p \left( \cdot \mid \theta_t^{(i)} \right)$ for $j=1,\dots,M$,
\begin{align*}
\E \Bigg[ \Bigg\| \frac{1}{M}\sum_{j=1}^M g \left( \tau_{t,j}^{(i)} \mid \theta_t^{(i)} \right) - \nabla J \left( \theta_t^{(i)} \right) \Bigg\|^2 \Bigg] \leq \frac{\sigma^2}{M}.
\end{align*}
\end{lemma}
\begin{proof}
Using independence of the gradient estimates in the first step, we get
\begin{align*}
\E \Bigg[ \Bigg\| \frac{1}{M}\sum_{j=1}^M g \left( \tau_{t,j}^{(i)} \mid \theta_t^{(i)} \right) - \nabla J \left( \theta_t^{(i)} \right) \Bigg\|^2 \Bigg] = \frac{1}{M^2}\sum_{j=1}^M \underbrace{\E \left[ \left\| g \left( \tau_{t,j}^{(i)} \mid \theta_t^{(i)} \right) - \nabla J \left( \theta_t^{(i)} \right) \right\|^2 \right]}_{\leq \;\sigma^2} \leq \frac{\sigma^2}{M}
\end{align*}
where the inequality in the last step is by Assumption \ref{ass:fin-var}.
\end{proof}

Finally, we give two lemmas that are slight variations of classical concentration bounds. The former is Markov's inequality, stated in terms of the random variable's second moment. The latter is a Chernoff bound for the case in which only an upper bound on the expectation is known.
\begin{lemma}[Second moment version of Markov's inequality]
\label{lem:markov}
Let $X$ be a non-negative random variable and $a>0$. Then
\begin{align*}
\Pr[X \geq a] \leq \frac{\E [X^2]}{a^2}.
\end{align*}
\end{lemma}
\begin{proof}
By Markov's inequality,
\begin{align*}
\Pr[X \geq a] = \Pr[X^2 \geq a^2] \leq \frac{\E [X^2]}{a^2}.
\end{align*}
\end{proof}

\begin{lemma}[Chernoff with bounded probabilities]
\label{lem:chernoff}
For $i \in [n]$, let $Y_i \stackrel{i.i.d.}{\sim}\text{Ber}(p_i)$ where $p_i \leq \hat{p} \in [0,1]$ for all $i \in [n]$. With $Y:=\sum_{i=1}^n Y_i$, we have $\E \left[ Y \right] \leq n\hat{p}$, and for any $\delta \geq 0$,
\begin{align*}
\Pr \left[ Y \geq (1+\delta) n\hat{p} \right] \leq \exp \left( \frac{-\delta^2 n \hat{p}}{2+\delta} \right).
\end{align*}
\end{lemma}
\begin{proof}
For $i \in [n]$, let $X_i \stackrel{i.i.d.}{\sim}\text{Ber}(\hat{p})$. We then have
\begin{align*}
\Pr \left[ Y \geq (1+\delta) n\hat{p} \right] &\stackrel{(a)}{\leq} \Pr \left[ X \geq (1+\delta) n\hat{p} \right] \\
&\stackrel{(b)}{\leq} \exp \left( \frac{-\delta^2 n \hat{p}}{2+\delta} \right)
\end{align*}
where (a) is because by definition, $\Pr[X_i=1] \geq \Pr[Y_i=1]$ for all $i \in [n]$, and (b) is the standard Chernoff bound applied to $X_1,\dots,X_n$.
\end{proof}

\newpage
\section{Proofs for Section \ref{sec:ByzPG}}
\label{app:ByzPG-proofs}

In this section, we prove Theorem \ref{thm:ByzPG} and Corollary \ref{cor:ByzPG}. We start with the main results, and then proceed with the required technical lemmas.

\subsection{Main Result}

\begin{proof}[Proof of Theorem \ref{thm:ByzPG}]
Following the strategy of \cite{gorbunov2023variance} applied to our policy gradient setting, we let
\begin{align*}
\Phi_t:=J^{*} - J \left( \theta_t \right)+\frac{\eta}{p} \left\| v_t-\nabla J \left( \theta_t \right) \right\|^2.
\end{align*}

Then, applying first Lemma \ref{lem:ByzPG-appl-smoothness}, followed by Lemma \ref{lem:ByzPG-distortion}, we derive
\begin{align*}
\E \left[ \Phi_{t+1} \right] &\leq \E \bigg[ J^{*}-J \left( \theta_t \right)-\frac{\eta}{2} \left\| \nabla J \left( \theta_t \right) \right\|^2 - \left( \frac{1}{2\eta}-\frac{L}{2} \right) \left\| \theta_{t+1}-\theta_t \right\|^2 + \frac{\eta}{2} \left\| v_t-\nabla J \left( \theta_t \right) \right\|^2 \\
&\quad+ \frac{\eta}{p} \left\| v_{t+1}-\nabla J \left( \theta_{t+1} \right) \right\|^2 \bigg]\\
&\leq \E \left[ J^{*}-J \left( \theta_t \right) - \left( \frac{1}{2\eta}-\frac{L}{2} \right) \left\| \theta_{t+1}-\theta_t \right\|^2 + \frac{\eta}{2} \left\| v_t-\nabla J \left( \theta_t \right) \right\|^2 \right] -\frac{\eta}{2} \E\left\| \nabla J \left( \theta_t \right) \right\|^2\\
&\quad+ \left( \frac{\eta}{p}-\frac{\eta}{2} \right) \E \left[ \left\| v_t-\nabla J \left( \theta_t \right) \right\|^2 \right] + \frac{\eta A}{2} \E \left[ \left\| \theta_{t+1}-\theta_t \right\|^2 \right] + \frac{\eta C \sigma^2}{2N} \left( \alpha+\frac{1}{K} \right)\\
&\leq \E \left[ \Phi_t \right] - \frac{\eta}{2} \E \left[ \left\| \nabla J \left( \theta_t \right) \right\|^2 \right] + \frac{\eta C \sigma^2}{2N} \left( \alpha+\frac{1}{K} \right)-\frac{1}{2\eta} \left( 1-L\eta-A\eta^2 \right) \E \left[ \left\| \theta_{t+1}-\theta_t \right\|^2 \right] \\
&\leq \E \left[ \Phi_t \right] - \frac{\eta}{2} \E \left[ \left\| \nabla J \left( \theta_t \right) \right\|^2 \right] + \frac{\eta C \sigma^2}{2N} \left( \alpha+\frac{1}{K} \right)
\end{align*}
where in the last step we use Lemma C.1 from \cite{gorbunov2023variance}, our choice of $\eta$, and the fact that $A=\Theta \left( p^{-1}K^{-1} \right)$, see Lemma \ref{lem:ByzPG-distortion}. Summing over $t$, we obtain
\begin{align*}
\frac{1}{T} \sum_{t=0}^{T-1} \E \left[ \left\| \nabla J \left( \theta_t \right) \right\|^2 \right] &\leq \frac{2}{\eta T} \sum_{t=0}^{T-1} \left( \E \left[
\Phi_t \right] - \left[ \Phi_{t+1} \right] \right) + \frac{C \sigma^2}{N} \left( \alpha+\frac{1}{K} \right) \\
&\leq \frac{2 \E \left[ \Phi_0 \right]}{\eta T} + \frac{C \sigma^2}{N} \left( \alpha+\frac{1}{K} \right)
\end{align*}
where in the second step we simplify the telescoping sum and use the fact that $\Phi_T \geq 0$. Note that the LHS is exactly $\E \left[ \left\| \nabla J \left( \theta_{\widehat{T}} \right) \right\|^2 \right]$ with $\widehat{T}$ chosen uniformly at random from $[T]$.
\end{proof}

\begin{proof}[Proof of Corollary \ref{cor:ByzPG}]
In order to achieve
\begin{align*}
\underbrace{\frac{2 \E \left[ \Phi_0 \right]}{\eta T}}_{(a)} + \underbrace{\frac{C \sigma^2}{N} \left( \alpha+\frac{1}{K} \right)}_{(b)} \leq \epsilon^2,
\end{align*}
we set
\begin{align*}
B=\Theta(1), \quad T=\Theta \left( \frac{1}{\epsilon^2 \sqrt{pK}} \right), \quad\text{and}\quad
N=\Theta \left( \epsilon^{-2} \left( \alpha+\frac{1}{K} \right) \right).
\end{align*}
Since $\eta=\Theta \left( \min \left\{ \sqrt{pK},1/L \right\} \right)$, above choice of $T$ ensures that term (a) is bounded by $\epsilon^2/2$. We treat $\sigma^2$ as constant, hence our choice of $N$ also ensures a $\epsilon^2/2$-bound for term (b).

To conclude the proof, we want to count the number of trajectories sampled at an agent: Per iteration, with probability $p$ we sample $N$ times, and with probability $1-p$ we sample $B$ times. Hence over all $T$ iterations, choosing $p=\frac{1}{N}=\frac{\epsilon^2}{\alpha+1/K}$, the expected number of sampled trajectories is given by
\begin{align*}
T \left( pN+(1-p)B \right) \,\leq\, \mathcal O(TpN) \stackrel{p=\frac{1}{N}}= \mathcal O \left( \epsilon^{-2}\frac{1}{\sqrt{pK}} \right) &= \mathcal O \left( \epsilon^{-2} \sqrt{\frac{\alpha+1/K}{K\epsilon^2}} \right) \\
&\leq \mathcal O \left( \frac{\alpha^{1/2}}{K^{1/2}\epsilon^3}+\frac{1}{K \epsilon^3} \right).
\end{align*}
\end{proof}

\subsection{Technical Lemmas}
Here, we give proofs of the technical lemmas used in above proofs. The first is a standard application of smoothness, also seen in a similar form e.g.\ in \cite{gargiani2022page,li2021page}.

\begin{lemma}\label{lem:ByzPG-appl-smoothness}
Let Assumptions \ref{ass:bounded-policy}, \ref{ass:smoothness}, \ref{ass:fin-var}, and \ref{ass:fin-imp-var} hold. Then, for $\theta_{t+1}=\theta_t+\eta v_t$,
\begin{align*}
J \left( \theta_{t+1} \right) \geq J \left( \theta_t \right) + \frac{\eta}{2} \left\| \nabla J \left( \theta_t \right) \right\|^2 + \left( \frac{1}{2\eta}-\frac{L}{2} \right)\left\| \theta_{t+1}-\theta_t \right\|^2 - \frac{\eta}{2} \left\| v_t-\nabla J \left( \theta_t \right) \right\|^2.
\end{align*}
\end{lemma}
\begin{proof}
By smoothness of $J$, i.e., Proposition \ref{prop:smoothness},
\begin{align*}
J \left( \theta_{t+1} \right) &\geq J \left( \theta_t \right) + \langle \nabla J \left( \theta_t \right), \theta_{t+1}-\theta_t \rangle - \frac{L}{2} \left\| \theta_{t+1}-\theta_t \right\|^2\\
&= J \left( \theta_t \right) + \eta \langle \nabla J \left( \theta_t \right), v_t \rangle - \frac{\eta^2 L}{2} \left\| v_t \right\|^2\\
&\stackrel{(a)}{=} J \left( \theta_t \right) + \frac{\eta}{2} \left\| \nabla J \left( \theta_t \right) \right\|^2 + \frac{\eta}{2} \left\| v_t \right\|^2 - \frac{\eta}{2} \left\| v_t-\nabla J \left( \theta_t \right) \right\|^2 - \frac{\eta^2L}{2} \left\| v_t \right\|^2\\
&= J \left( \theta_t \right) + \frac{\eta}{2} \left\| \nabla J \left( \theta_t \right) \right\|^2 + \left( \frac{1}{2\eta}-\frac{L}{2} \right)\left\| \theta_{t+1}-\theta_t \right\|^2 - \frac{\eta}{2} \left\| v_t-\nabla J \left( \theta_t \right) \right\|^2
\end{align*}
where (a) is due to Lemma \ref{lem:basics}, (\ref{eqn:scalar-prod}).
\end{proof}

As a next preparatory step, we want to show that the gradient estimates $\tilde{v}_t^{(i)}$ do not vary too much across different agents. In particular, we bound the average of pairwise variances, denoted by $\widetilde{\mathcal T}$. Note that we slightly abuse notation and omit adding an index $t$ to $\widetilde{\mathcal T}$ since the bound below does not depend on $t$.
\begin{lemma}\label{lem:ByzPG-variance-bound}
Let Assumption \ref{ass:byz-fraction} and \ref{ass:fin-var} hold, i.e., $\mathcal H_t \subset K$ is the set of honest agents in iteration $t$. Then,
\begin{align*}
\widetilde{\mathcal T} &:= \frac{1}{|\mathcal H_t|(|\mathcal H_t|-1)} \sum_{i,l \in \mathcal H_t} \E \left[ \left\| \tilde{v}_t^{(i)}-\tilde{v}_t^{(l)} \right\|^2 \mid c_t=1 \right] \leq \frac{4\sigma^2}{N}
\end{align*}
and we will use $\widetilde{\mathcal T}$ also in subsequent lemmas to refer to this term.
\end{lemma}
\begin{proof}
By plugging in the definition of $\tilde{v}_t^{(i)}$ as in Algorithm \ref{alg:ByzPG}, and using basic fact (\ref{eqn:norm-of-sum}) from Lemma \ref{lem:basics}, we get
\begin{align*}
\widetilde{\mathcal T} &= \frac{1}{|\mathcal H_t|(|\mathcal H_t|-1)} \sum_{i,l \in \mathcal H_t} \E \left[ \left\| \tilde{v}_t^{(i)}-\tilde{v}_t^{(l)} \right\|^2 \mid c_t=1 \right] \\
&= \frac{1}{|\mathcal H_t|(|\mathcal H_t|-1)} \sum_{i,l \in \mathcal H_t} \E \left[ \left\| \frac{1}{N} \sum_{j=1}^N g \left( \tau_{t,j}^{(i)} \mid \theta_t \right) - \frac{1}{N} \sum_{j=1}^N g \left( \tau_{t,j}^{(l)} \mid \theta_t \right) \right\|^2 \right] \\
&\leq \frac{1}{|\mathcal H_t|(|\mathcal H_t|-1)} \sum_{i,l \in \mathcal H_t} \E \left[ 2\left\| \frac{1}{N} \sum_{j=1}^N g \left( \tau_{t,j}^{(i)} \mid \theta_t \right) - \nabla J \left( \theta_t \right) \right\|^2 + 2\left\| \frac{1}{N} \sum_{j=1}^N g \left( \tau_{t,j}^{(l)} \mid \theta_t \right) - \nabla J \left( \theta_t \right) \right\|^2 \right] \\
&= \frac{4}{|\mathcal H_t|} \sum_{i \in \mathcal H_t} \underbrace{\E \left[ \left\| \frac{1}{N}\sum_{j=1}^N g \left( \tau_{t,j}^{(i)} \mid \theta_t \right) - \nabla J \left( \theta_t \right) \right\|^2 \right]}_{\stackrel{(a)}{\leq} \;\frac{\sigma^2}{N}} \\
&\leq \frac{4\sigma^2}{N}
\end{align*}
where (a) is by Lemma \ref{lem:indep-sampling-variance}.
\end{proof}

Next, we want to bound the deviation of our potentially aggregated gradient estimate $v_t$ from $\nabla J \left( \theta_t \right)$ which is what we aim to estimate. Note that in case $c_t=1$, this involves bounding both the bias introduced by aggregation, as well as the variance due to sampling.
\begin{lemma}
\label{lem:ByzPG-distortion-aux}
In the setting of Theorem \ref{thm:ByzPG}, the following bounds hold:
\begin{align*}
\mathcal T_{1,t}^{(1)} &:= \E \left[ \left\| v_t-\bar{\tilde{v}}_t \right\|^2 \mid c_t=1 \right] \leq C_{ra}\alpha \widetilde{\mathcal T}, \\
\mathcal T_{1,t}^{(2)} &:= \E \left[ \left\| \bar{\tilde{v}}_t-\nabla J \left( \theta_t \right) \right\|^2 \mid c_t=1 \right] \leq \frac{2\sigma^2}{KN}, \text{ and} \\
\mathcal T_{0,t}^{(1)} &:= \E \left[ \left\| v_t-\nabla J \left( \theta_t \right) \right\|^2 \mid c_t=0 \right] \leq \E \left[ \left\| v_{t-1}-\nabla J \left( \theta_{t-1} \right) \right\|^2 \right] + \left( \frac{C_p}{BK} + \frac{L^2}{K} \right) \E \left[ \left\| \theta_t - \theta_{t-1} \right\|^2 \right],
\end{align*}
where $\mathcal T_{1,t}^{(1)}$, $\mathcal T_{1,t}^{(2)}$ are the errors w.r.t.\ aggregation and sampling, respectively, in case $c_t=1$, and $\mathcal T_{0,t}^{(1)}$ is the error due to sampling in case $c_t=0$ (no aggregation done in this case). We will use this notation also to refer to the respective terms in subsequent lemmas.
\end{lemma}
\begin{proof}
We proceed by showing each of the claimed bounds separately, as follows:
\begin{enumerate}
\item First, by applying Definition \ref{def:aggregate},
\begin{align*}
\mathcal T_{1,t}^{(1)}
&= \E \left[ \left\| \bar{v}_t-\bar{\tilde{v}}_t \right\|^2 \mid c_t=1 \right] \\
&\leq \frac{C_{ra} \alpha}{|\mathcal H_t|(|\mathcal H_t|-1)} \sum_{i,l \in \mathcal H_t} \E \left[ \left\| \tilde{v}_t^{(i)}-\tilde{v}_t^{(l)} \right\|^2 \mid c_t=1 \right] \\
&\leq C_{ra}\alpha \widetilde{\mathcal T}
\end{align*}
where the second step is the result of Lemma \ref{lem:ByzPG-variance-bound}.

\item Next, by plugging in the definition of $\bar{\tilde{v}}_t$ and then applying Lemma \ref{lem:indep-sampling-variance}, we get
\begin{align*}
\mathcal T_{1,t}^{(2)} &= \E \left[ \left\| \bar{\tilde{v}}_t-\nabla J \left( \theta_t \right) \right\|^2 \mid c_t=1 \right] \\
&\leq \E \left[ \left\| \frac{1}{|\mathcal H_t|N} \sum_{i \in \mathcal H_t} \sum_{j=1}^N g \left( \tau_{t,j}^{(i)} \mid \theta_t \right)-\nabla J \left( \theta_t \right) \right\|^2 \right] \\
&\leq \frac{\sigma^2}{|\mathcal H_t|N} \\
&\leq \frac{2\sigma^2}{KN}
\end{align*}
where in the final step we use the fact that according to Assumption \ref{ass:byz-fraction}, we have $\alpha_{\max} \leq 1/2$ and therefore $|\mathcal H_t| \geq (1-\alpha)K \geq (1-\alpha_{\max})K \geq K/2$.

\item We have
\begin{align*}
\mathcal T_{0,t}^{(1)}&=\E \left[ \left\| v_t - \nabla J \left( \theta_t \right) \right\|^2 \mid c_t=0 \right] \\
&=\E \left[ \left\| \hat{\Delta}^B \left( \theta_t,\theta_{t-1} \right) + v_{t-1} - \nabla J \left( \theta_t \right) \right\|^2 \right] \\
&=\E \left[ \left\| \hat{\Delta}^B \left( \theta_t,\theta_{t-1} \right) + v_{t-1} + \nabla J \left( \theta_{t-1} \right) - \nabla J \left( \theta_{t-1} \right) - \nabla J \left( \theta_t \right) \right\|^2 \right] \\
&=\E \Bigg[ \Big\| \underbrace{\left( v_{t-1}- \nabla J \left( \theta_{t-1} \right) \right)}_{=: X_1} + \underbrace{\left( \hat{\Delta}^B \left( \theta_t,\theta_{t-1} \right) - \Delta \left( \theta_t,\theta_{t-1} \right) \right)}_{=:X_2} \Big\|^2 \Bigg] \\
&\stackrel{(a)}{=} \E \left[ \left\| v_{t-1}- \nabla J \left( \theta_{t-1} \right) \right\|^2 \right] + \E \left[ \left\| \hat{\Delta}^B \left( \theta_t,\theta_{t-1} \right) - \Delta \left( \theta_t,\theta_{t-1} \right) \right\|^2 \right] \\
&\stackrel{(b)}{\leq} \E \left[ \left\| v_{t-1}-\nabla J \left( \theta_{t-1} \right) \right\|^2 \right] + \left( \frac{C_p}{BK} + \frac{L^2}{K} \right) \E \left[ \left\| \theta_t - \theta_{t-1} \right\|^2 \right]
\end{align*}
Steps (a) and (b) require further justification but appear again equivalently in the last part of the proof of Lemma \ref{lem:DecByzPG-distortion} where we explain the details.
\end{enumerate}
\end{proof}

Finally, we combine the individual bounds of Lemma \ref{lem:ByzPG-distortion-aux} into an overall distortion bound, as follows.
\begin{lemma}\label{lem:ByzPG-distortion}
In the setting of Lemma \ref{lem:ByzPG-distortion-aux}, we have
\begin{align*}
\E \left\| v_t-\nabla J \left( \theta_t \right) \right\|^2 &\leq \left( 1-p \right) \E \left[ \left\| v_{t-1}-\nabla J \left( \theta_{t-1} \right) \right\|^2 \right] + \frac{Ap}{2} \E \left[ \left\| \theta_t - \theta_{t-1} \right\|^2 \right] + \frac{C p \sigma^2}{2N} \left( \alpha+\frac{1}{K} \right)
\end{align*}
where $C>0$ is a constant, and $A = \Theta \left( \frac{1}{pK} \right)$.
\end{lemma}
\begin{proof}
By the law of total expectation and using Lemma \ref{lem:basics}, (\ref{eqn:norm-of-sum}) in the first step, we have
\begin{align*}
\E \left[ \left\| v_t-\nabla J \left( \theta_t \right) \right\|^2 \right] &\leq \left( 1-p \right) \mathcal T_{0,t}^{(1)} + p \left( 2\mathcal T_{1,t}^{(1)}+2\mathcal T_{1,t}^{(2)} \right) \\
&\leq \left( 1-p \right) \left( \E \left[ \left\| v_{t-1}-\nabla J \left( \theta_{t-1} \right) \right\|^2 \right] + \left( \frac{2 C_p}{BK} + \frac{2L^2}{K} \right) \E \left[ \left\| \theta_t - \theta_{t-1} \right\|^2 \right] \right) \\
&\quad+\frac{8pC_{ra}\alpha\sigma^2}{N} + \frac{4p\sigma^2}{KN}.
\end{align*}
where in the second step we have plugged in the results of Lemma \ref{lem:ByzPG-distortion-aux}.
Since $1-p \leq 1, B \geq 1$, and treating $C_p,L,$ and $C_{ra}$ as constant, the result follows.
\end{proof}

\section{Proofs for Section \ref{sec:DecByzPG}}
\label{app:DecByzPG-proofs}
First, we define (or recall from previous definitions), that for any $k \in [K]$:
\begin{align*}
\hat{v}^{(k)}_t &:= \frac{1}{\eta} \left( \theta_{t+1}^{(k)}-\theta_t^{(k)} \right), \\
\bar{\hat{v}}_t &:= \frac{1}{|\mathcal G_t|} \sum_{i \in \mathcal G_t} \hat{v}^{(i)}_t, \quad\quad \bar{\tilde{v}}_t := \frac{1}{|\mathcal G_t|} \sum_{i \in \mathcal G_t} \tilde{v}^{(i)}_t \\
\bar\theta_t &:= \frac{1}{|\mathcal G_t|} \sum_{i \in \mathcal G_t} \theta^{(i)}_t, \quad\quad \bar{\tilde{\theta}}_t := \frac{1}{|\mathcal G_t|} \sum_{i \in \mathcal G_t} \tilde{\theta}^{(i)}_t.
\end{align*}

In the remainder of this section, we prove Theorem \ref{thm:DecByzPG} and Corollary \ref{cor:DecByzPG}. As in Section \ref{app:ByzPG-proofs}, we first show the main results, and then provide proofs of the required technical lemmas.

\subsection{Main Result}

\begin{proof}[Proof of Theorem \ref{thm:DecByzPG}]
Following the proof strategy from \cite{gorbunov2023variance} applied to our decentralized policy gradient setting, we let
\begin{align*}
\Phi_t:=J^{*} - J \left( \bar{\theta}_t \right)+\frac{2\eta}{p} \Bigg\| \bar{\hat{v}}_t-\frac{1}{|\mathcal G_t|}\sum_{i \in \mathcal G_t}\nabla J \left( \theta_t^{(i)} \right) \Bigg\|^2.
\end{align*}

Then, applying first Lemma \ref{lem:DecByzPG-appl-smoothness} and then Lemma \ref{lem:DecByzPG-distortion}, we derive
\begin{align*}
\E \left[ \Phi_{t+1} \right] &\leq \E \Biggl[ J^{*}-J \left( \bar{\theta}_t \right)-\frac{\eta}{2} \left\| \nabla J \left( \bar{\theta}_t \right) \right\|^2 - \left( \frac{1}{2\eta}-\frac{L}{2} \right) \left\| \bar{\theta}_{t+1}-\bar{\theta}_t \right\|^2 + \eta \biggl\| \bar{\hat{v}}_t-\frac{1}{|\mathcal G_t|}\sum_{i \in \mathcal G_t}\nabla J \left( \theta_t^{(i)} \right) \biggr\|^2 \\
&\quad+ \eta \mathcal E^{\Delta}_t L^2 + \frac{2\eta}{p} \biggl\| \bar{\hat{v}}_{t+1}-\frac{1}{|\mathcal G_t|}\sum_{i \in \mathcal G_t}\nabla J \left( \theta_{t+1}^{(i)} \right) \biggr\|^2 \Biggr]\\
&\leq \E \left[ J^{*}-J \left( \bar{\theta}_t \right) - \left( \frac{1}{2\eta}-\frac{L}{2} \right) \Big\| \bar{\theta}_{t+1}-\bar{\theta}_t \Big\|^2 + \eta \Bigg\| \bar{\hat{v}}_t-\frac{1}{|\mathcal G_t|}\sum_{i \in \mathcal G_t}\nabla J \left( \theta_t^{(i)} \right) \Bigg\|^2 \right] + \eta L^2 \E[\mathcal E^{\Delta}_t] \\
&\quad- \frac{\eta}{2} \E \left[ \left\| \nabla J \left( \bar{\theta}_t \right) \right\|^2 \right] + \frac{2\eta}{p} \left( 1-\frac{p}{2} \right) \E \left[ \Bigg\| \bar{\hat{v}}_t-\frac{1}{|\mathcal G_t|}\sum_{i \in \mathcal G_t}\nabla J \left( \theta_t^{(i)} \right) \Bigg\|^2 \right] + \frac{2\eta Ap}{4p} \E \left[ \left\| \bar{\theta}_{t+1}-\bar{\theta}_t \right\|^2 \right] \\
&\quad+ \frac{2\eta C_1\sigma^2p\alpha}{4Np} + \frac{2\eta C_2\sigma^2p}{4KNp} \\
&\leq \E \left[ \Phi_t \right] - \frac{\eta}{2} \E \left[ \left\| \nabla J \left( \bar{\theta}_t \right) \right\|^2 \right] + \frac{\eta C_1\sigma^2\alpha}{2N} + \frac{\eta C_2\sigma^2}{2KN}-\frac{1}{2\eta} \left( 1-L\eta-A\eta^2 \right) \E \left[ \left\| \bar{\theta}_{t+1}-\bar{\theta}_t \right\|^2 \right] + \mathcal O \left( \eta 2^{-\kappa} \right)\\
&\leq \E \left[ \Phi_t \right] - \frac{\eta}{2} \E \left[ \left\| \nabla J \left( \bar{\theta}_t \right) \right\|^2 \right] + \frac{\eta C_1\sigma^2\alpha}{2N} + \frac{\eta C_2\sigma^2}{2KN} + \mathcal O \left( \eta 2^{-\kappa} \right)
\end{align*}
where in the second to last step we use Lemma \ref{lem:diam-error} and in last step we use Lemma C.1 from \cite{gorbunov2023variance} and our choice of $\eta$. We now sum over $t$ to obtain
\begin{align*}
\frac{1}{T} \sum_{t=0}^{T-1} \E \left[ \left\| \nabla J \left( \bar{\theta}_t \right) \right\|^2 \right] &\leq \frac{2}{\eta T} \sum_{t=0}^{T-1} \left( \E \left[
\Phi_t \right] - \left[ \Phi_{t+1} \right] \right) + \frac{C_1\sigma^2\alpha}{N} + \frac{C_2\sigma^2}{KN}+\mathcal O \left( 2^{-\kappa} \right)\\
&\leq \frac{2 \E \left[ \Phi_0 \right]}{\eta T} + \frac{C\sigma^2}{2N} \left( \alpha+\frac{1}{K} \right)+\mathcal O \left( 2^{-\kappa} \right)
\end{align*}
where $C:= \frac{1}{2}\max \left( C_1,C_2 \right)$. In the second step, we simplify the telescoping sum and use the fact that $\Phi_T \geq 0$. Note that the LHS is exactly $\E \left[ \left\| \nabla J \left( \bar{\theta}_{\widehat{T}} \right) \right\|^2 \right]$ with $\widehat{T}$ chosen u.a.r.\ from $[T]$. To finish the proof, notice that for any $k \in \mathcal G_t$,
\begin{align*}
\E \left[ \left\| \nabla J \left( \theta^{(k)}_{\widehat{T}} \right) \right\|^2 \right]
&\stackrel{(a)}{\leq} 2\E \left[ \left\| \nabla J \left( \bar{\theta}_{\widehat{T}} \right) \right\|^2 \right] + 2\E \left[ \left\| \nabla J \left( \theta^{(k)}_{\widehat{T}} \right)-\nabla J \left( \bar{\theta}_{\widehat{T}} \right) \right\|^2 \right] \\
&\stackrel{(b)}{\leq} 2\E \left[ \left\| \nabla J \left( \bar{\theta}_{\widehat{T}} \right) \right\|^2 \right] + 2L^2\E \left[ \mathcal E^\Delta_{\widehat{T}} \right] \\
&\stackrel{(c)}{\leq} \frac{4 \E \left[ \Phi_0 \right]}{\eta T} + \frac{C\sigma^2}{N} \left( \alpha+\frac{1}{K} \right)+\mathcal O \left( 2^{-\kappa} \right)
\end{align*}
where (a) is by Lemma \ref{lem:basics}, (\ref{eqn:norm-of-sum}), (b) easily follows from $L$-smoothness of $J(\cdot)$ (i.e.\ Proposition \ref{prop:smoothness}) and (\ref{eqn:diam-bound}), and (c) is by Lemma \ref{lem:diam-error}.
\end{proof}

\begin{proof}[Proof of Corollary \ref{cor:DecByzPG}]
In order to achieve
\begin{align*}
\underbrace{\frac{4 \E \left[ \Phi_0 \right]}{\eta T}}_{(a)} + \underbrace{\frac{C\sigma^2}{N} \left( \alpha+\frac{1}{K} \right)}_{(b)} + \underbrace{\mathcal O \left( 2^{-\kappa} \right)}_{(c)} \leq \epsilon^2,
\end{align*}
we set
\begin{align*}
B=\Theta(1), \quad T=\Theta \left( \epsilon^{-2} \cdot \max \left\{ \frac{\sqrt{\alpha}}{p}, \frac{1}{\sqrt{pK}} \right\} \right), \quad\text{and}\quad
N=\Theta \left( \epsilon^{-2} \left( \alpha+\frac{1}{K} \right) \right).
\end{align*}
Since, in the statement of Theorem \ref{thm:DecByzPG}, we require
\begin{align*}
\eta \leq \mathcal O \left( \frac{1}{\sqrt{A}} \right) = \mathcal O \left( \min \left\{ \frac{p}{\sqrt{\alpha}},\sqrt{pK} \right\} \right),
\end{align*}
above choice of $T$ can ensure that term (a) is at most $\epsilon^2/3$. We treat $C_{ra},C_1,C_2,\sigma^2$ as constants, hence our choice of $N$ also ensures a $\epsilon^2/3$-bound for term (b), and the same holds for (c) due to our choice of $\kappa \leq \Theta \left( \log \epsilon^{-1} \right)$.

To conclude the proof, we want to count the number of trajectories sampled at an agent: Per iteration, with probability $p$ we sample $N$ times, and with probability $1-p$ we sample $B$ times. Hence over all $T$ iterations, choosing $p=\frac{1}{N}=\frac{\epsilon^2}{\alpha+1/K}$, the expected number of sampled trajectories is given by
\begin{align*}
T \left( pN+(1-p)B \right) &\leq \mathcal O(TpN)\\
&\stackrel{p=\frac{1}{N}}= \mathcal O \left( \epsilon^{-2} \cdot \max \left\{ \frac{\sqrt{\alpha}}{p}, \frac{1}{\sqrt{pK}} \right\} \right)\\
&\leq \mathcal O \left( \epsilon^{-2} \left( \frac{\alpha^{1/2} \left( \alpha+1/K \right)}{\epsilon^2} + \frac{\sqrt{\alpha+1/K}}{K^{1/2}\epsilon} \right) \right)\\
&\leq \mathcal O \left( \frac{\alpha^{3/2}}{\epsilon^4}+\frac{\alpha^{1/2}}{K\epsilon^4}+\frac{\alpha^{1/2}}{K^{1/2}\epsilon^3}+\frac{1}{K\epsilon^3} \right).
\end{align*}
\end{proof}

\subsection{Technical Lemmas}

Similar to Lemma \ref{lem:ByzPG-appl-smoothness} in the centralized case, we start with a standard result following from smoothness, now adjusted to the decentralized setting.
\begin{lemma}\label{lem:DecByzPG-appl-smoothness}
Under Assumptions \ref{ass:bounded-policy}, \ref{ass:smoothness}, \ref{ass:fin-var}, and \ref{ass:fin-imp-var}, and for $\bar{\hat{v}}_t := \frac{1}{\eta} \left( \bar{\theta}_{t+1}-\bar{\theta}_t \right)$, we have
\begin{align*}
J \left( \bar{\theta}_{t+1} \right) \geq J \left( \bar{\theta}_t \right) + \frac{\eta}{2} \left\| \nabla J \left( \bar{\theta}_t \right) \right\|^2 + \left( \frac{1}{2\eta}-\frac{L}{2} \right)\left\| \bar{\theta}_{t+1}-\bar{\theta}_t \right\|^2 - \eta \left\| \bar{\hat{v}}_t-\frac{1}{|\mathcal G_t|}\sum_{i \in \mathcal G_t}\nabla J \left( \theta_t^{(i)} \right) \right\|^2 + \eta L^2 \mathcal E^{\Delta}_t.
\end{align*}
\end{lemma}
\begin{proof}
By smoothness of $J$, i.e., Proposition \ref{prop:smoothness},
\begin{align*}
J \left( \bar{\theta}_{t+1} \right) &\geq J \left( \bar{\theta}_t \right) + \langle \nabla J \left( \bar{\theta}_t \right), \bar{\theta}_{t+1}-\bar{\theta}_t \rangle - \frac{L}{2} \left\| \bar{\theta}_{t+1}-\bar{\theta}_t \right\|^2\\
&= J \left( \bar{\theta}_t \right) + \eta \langle \nabla J \left( \bar{\theta}_t \right), \bar{\hat{v}}_t \rangle - \frac{\eta^2 L}{2} \left\| \bar{\hat{v}}_t \right\|^2\\
&\stackrel{(a)}{=} J \left( \bar{\theta}_t \right) + \frac{\eta}{2} \left\| \nabla J \left( \bar{\theta}_t \right) \right\|^2 + \frac{\eta}{2} \left\| \bar{\hat{v}}_t \right\|^2 - \frac{\eta}{2} \left\| \bar{\hat{v}}_t-\nabla J \left( \bar{\theta}_t \right) \right\|^2 - \frac{\eta^2L}{2} \left\| \bar{\hat{v}}_t \right\|^2\\
&= J \left( \bar{\theta}_t \right) + \frac{\eta}{2} \left\| \nabla J \left( \bar{\theta}_t \right) \right\|^2 + \left( \frac{1}{2\eta}-\frac{L}{2} \right)\left\| \bar{\theta}_{t+1}-\bar{\theta}_t \right\|^2 - \frac{\eta}{2} \left\| \bar{\hat{v}}_t-\nabla J \left( \bar{\theta}_t \right) \right\|^2 \\
&\stackrel{(b)}{=} J \left( \bar{\theta}_t \right) + \frac{\eta}{2} \left\| \nabla J \left( \bar{\theta}_t \right) \right\|^2 + \left( \frac{1}{2\eta}-\frac{L}{2} \right)\left\| \bar{\theta}_{t+1}-\bar{\theta}_t \right\|^2 - \eta \left\| \bar{\hat{v}}_t-\frac{1}{|\mathcal G_t|}\sum_{i \in \mathcal G_t}\nabla J \left( \theta_t^{(i)} \right) \right\|^2 \\
&\quad+ \eta \left\| \nabla J \left( \bar{\theta}_t \right)- \frac{1}{|\mathcal G_t|}\sum_{i \in \mathcal G_t}\nabla J \left( \theta_t^{(i)} \right) \right\|^2 \\
&\stackrel{(c)}{=} J \left( \bar{\theta}_t \right) + \frac{\eta}{2} \left\| \nabla J \left( \bar{\theta}_t \right) \right\|^2 + \left( \frac{1}{2\eta}-\frac{L}{2} \right)\left\| \bar{\theta}_{t+1}-\bar{\theta}_t \right\|^2 - \eta \left\| \bar{\hat{v}}_t-\frac{1}{|\mathcal G_t|}\sum_{i \in \mathcal G_t}\nabla J \left( \theta_t^{(i)} \right) \right\|^2 + \eta L^2 \mathcal E^{\Delta}_t
\end{align*}
where (a) is due to Lemma \ref{lem:basics},(\ref{eqn:scalar-prod}), (b) is by Lemma \ref{lem:basics}, (\ref{eqn:norm-of-sum}), and (c) easily follows from $L$-smoothness of $J$, i.e., Proposition \ref{prop:smoothness}, and the error bound (\ref{eqn:diam-bound}).
\end{proof}

As a next preparatory step, we want to show that the overall gradient estimates $\tilde{v}_t^{(i)}$ do not vary too much across different agents. In particular, we bound the average of pairwise variances as follows.
\begin{lemma}\label{lem:DecByzPG-variance-bound}
Let Assumption \ref{ass:byz-fraction} and \ref{ass:fin-var} hold. Then, there exists a constant $C_p>0$ such that
\begin{align*}
\widetilde{\mathcal T}_{1,t} &:= \frac{1}{|\mathcal G_t|(|\mathcal G_t|-1)} \sum_{i,l \in \mathcal G_t} \E \left[ \left\| \tilde{v}_t^{(i)}-\tilde{v}_t^{(l)} \right\|^2 \mid c_t=1 \right] \leq 8 \left( \frac{\sigma^2}{N} + L^2\E[\mathcal E^{\Delta}_t] \right) \\
\widetilde{\mathcal T}_{0,t} &:= \frac{1}{|\mathcal G_t|(|\mathcal G_t|-1)} \sum_{i,l \in \mathcal G_t} \E \left[ \left\| \tilde{v}_t^{(i)}-\tilde{v}_t^{(l)} \right\|^2 \mid c_t=0 \right] \leq \frac{12\overline{\mathcal E}^{\Delta,\kappa}}{\eta^2} + \frac{36 C_p \E\left[\mathcal E^{\Delta}_t\right]}{B} + \frac{18 C_p}{B} \E \left[ \left\| \bar{\theta}_t-\bar{\theta}_{t-1} \right\|^2 \right]
\end{align*}
and we will use $\widetilde{\mathcal T}_{1,t},\widetilde{\mathcal T}_{0,t}$ also in subsequent lemmas to refer to these terms.
\end{lemma}
\begin{proof}
For the first bound, we plug in the definition of $\tilde{v}_t^{(i)}$ as in Algorithm \ref{alg:ByzPG}, and use basic fact (\ref{eqn:norm-of-sum}) from Lemma \ref{lem:basics}, to get
\begin{align*}
\widetilde{\mathcal T}_{1,t} &= \frac{1}{|\mathcal G_t|(|\mathcal G_t|-1)} \sum_{i,l \in \mathcal G_t} \E \left[ \left\| \tilde{v}_t^{(i)}-\tilde{v}_t^{(l)} \right\|^2 \mid c_t=1 \right] \\
&= \frac{1}{|\mathcal G_t|(|\mathcal G_t|-1)} \sum_{i,l \in \mathcal G_t} \E \left[ \left\| \frac{1}{N} \sum_{j=1}^N g \left( \tau_{t,j}^{(i)} \mid \theta^{(i)}_t \right) - \frac{1}{N} \sum_{j=1}^N g \left( \tau_{t,j}^{(l)} \mid \theta^{(l)}_t \right) \right\|^2 \right] \\
&\leq \frac{1}{|\mathcal G_t|(|\mathcal G_t|-1)} \sum_{i,l \in \mathcal G_t} \E \left[ 2\left\| \frac{1}{N} \sum_{j=1}^N g \left( \tau_{t,j}^{(i)} \mid \theta^{(i)}_t \right) - \nabla J \left( \bar\theta_t \right) \right\|^2 + 2\left\| \frac{1}{N} \sum_{j=1}^N g \left( \tau_{t,j}^{(l)} \mid \theta^{(l)}_t \right) - \nabla J \left( \bar\theta_t \right) \right\|^2 \right] \\
&= \frac{4}{|\mathcal G_t|} \sum_{i \in \mathcal G_t} \E \left[ \left\| \frac{1}{N} \sum_{j=1}^N g \left( \tau_{t,j}^{(i)} \mid \theta^{(i)}_t \right) - \nabla J \left( \bar\theta_t \right) \right\|^2 \right] \\
&\leq \frac{4}{|\mathcal G_t|} \sum_{i \in \mathcal G_t} 2\underbrace{\E \left[ \left\| \frac{1}{N}\sum_{j=1}^N g \left( \tau_{t,j}^{(i)} \mid \theta_t^{(i)} \right) - \nabla J \left( \theta_t^{(i)} \right) \right\|^2 \right]}_{\stackrel{(a)}{\leq} \;\frac{\sigma^2}{N}} + 2\underbrace{\E \left[\left\|  \nabla J \left( \theta_t^{(i)} \right) - \nabla J \left( \bar{\theta}_t \right) \right\|^2 \right]}_{\stackrel{(b)}{\leq} \;L^2 \E\left[\mathcal E^{\Delta}_t\right]} \\
&\leq 8 \left( \frac{\sigma^2}{N} + L^2 \E\left[\mathcal E^{\Delta}_t\right] \right)
\end{align*}
where (a) is by Lemma \ref{lem:indep-sampling-variance} and (b) easily follows from $L$-smoothness of $J$, i.e., Proposition \ref{prop:smoothness}, and error bound (\ref{eqn:diam-bound}).

Then, the second bound follows from
\begin{align*}
\widetilde{\mathcal T}_{0,t} &= \frac{1}{|\mathcal G_t|(|\mathcal G_t|-1)} \sum_{i,l \in \mathcal G_t} \E \left[ \left\| \tilde{v}_t^{(i)}-\tilde{v}_t^{(l)} \right\|^2 \mid c_t=0 \right] \\
&= \frac{1}{|\mathcal G_t|(|\mathcal G_t|-1)} \sum_{i,l \in \mathcal G_t} \E \left[ \left\| \left( \hat{\Delta}^B \left( \theta_t^{(i)},\theta_{t-1}^{(i)} \right) + \hat{v}_{t-1}^{(i)} \right) - \left( \hat{\Delta}^B \left( \theta_t^{(l)},\theta_{t-1}^{(l)} \right) + \hat{v}_{t-1}^{(l)} \right) \right\|^2 \right] \\
&\leq \frac{1}{|\mathcal G_t|(|\mathcal G_t|-1)} \sum_{i,l \in \mathcal G_t} 3\,\E \left[ \left\| \hat{\Delta}^B \left( \theta_t^{(i)},\theta_{t-1}^{(i)} \right) \right\|^2 \right] + 3\,\E \left[ \left\| \hat{\Delta}^B \left( \theta_t^{(l)},\theta_{t-1}^{(l)} \right) \right\|^2 \right] + 3\,\E \left[ \left\| \hat{v}_{t-1}^{(i)} -  \hat{v}_{t-1}^{(l)} \right\|^2 \right] \\
&\stackrel{(a)}{\leq} \frac{12\overline{\mathcal E}^{\Delta,\kappa}}{\eta^2} + \frac{1}{|\mathcal G_t|(|\mathcal G_t|-1)} \sum_{i,l \in \mathcal G_t} \frac{3C_p}{B} \E \left[ \left\| \theta_t^{(i)}-\theta_{t-1}^{(i)} \right\|^2 \right] + \frac{3C_p}{B} \E \left[ \left\| \theta_t^{(l)}-\theta_{t-1}^{(l)} \right\|^2 \right] \\
&\leq \frac{12\overline{\mathcal E}^{\Delta,\kappa}}{\eta^2} + \frac{6C_p}{B|\mathcal G_t|}\sum_{i \in \mathcal G_t} \E \left[ \left\| \theta_t^{(i)}-\theta_{t-1}^{(i)} \right\|^2 \right] \\
&\leq \frac{12 \overline{\mathcal E}^{\Delta,\kappa}}{\eta^2} + \frac{6C_p}{B|\mathcal G_t|}\sum_{i \in \mathcal G_t} \E \left[ 3 \left\| \theta_t^{(i)}-\bar{\theta}_t \right\|^2 + 3 \left\| \bar{\theta}_t-\bar{\theta}_{t-1} \right\|^2 + 3 \left\| \bar{\theta}_{t-1}-\theta_{t-1}^{(i)} \right\|^2 \right] \\
&\stackrel{(b)}{\leq} \frac{12\overline{\mathcal E}^{\Delta,\kappa}}{\eta^2} + \frac{36 C_p \E \left[ \mathcal E^\Delta_t \right]}{B} + \frac{18 C_p}{B} \E \left[ \left\| \bar{\theta}_t-\bar{\theta}_{t-1} \right\|^2 \right].
\end{align*}
Step (a) is due to
\begin{align*}
\frac{1}{|\mathcal G_t| \left( |\mathcal G_t|-1 \right)} \sum _{i,l \in \mathcal G_t} \E \left[ \left\| \hat{v}_{t-1}^{(i)} -  \hat{v}_{t-1}^{(l)} \right\|^2 \right] &\leq \frac{1}{|\mathcal G_t| \left( |\mathcal G_t|-1 \right)} \sum _{i,l \in \mathcal G_t} \E \left[ \left\| \frac{1}{\eta} \left[ \left( \theta_t^{(i)}-\theta_t^{(l)} \right) + \left( \theta_{t-1}^{(i)}-\theta_{t-1}^{(l)} \right) \right] \right\|^2 \right] \\
&\leq \frac{2 \E \left[ \mathcal E^\Delta_t \right] + 2 \E \left[ \mathcal E^\Delta_{t-1} \right]}{\eta^2} \\
&\leq \frac{4 \overline{\mathcal E}^{\Delta,\kappa}}{\eta^2}
\end{align*}
where in the last step we use our bound $\overline{\mathcal E}^{\Delta,\kappa}$ from Lemma \ref{lem:diam-error} that is independent of $t$. The second fact used in (a) is that there exists constant $C_p>0$ such that
\begin{align*}
\E \left[ \left\| \hat{\Delta}^B \left( \theta_t^{(i)},\theta_{t-1}^{(i)} \right) \right\|^2 \right] \leq \frac{C_p}{B} \E \left[ \left\| \theta_t^{(i)}-\theta_{t-1}^{(i)} \right\|^2 \right]
\end{align*}
which follows similarly as Lemma B.1 in \cite{gargiani2022page}.
Step (b) is because, using Lemma \ref{lem:avg-vec-bound} in the first step, we have
\begin{align*}
\frac{1}{|\mathcal G_t|} \sum _{i \in \mathcal G_t} \left\| \theta_t^{(i)}-\bar{\theta}_t \right\|^2 \leq \frac{1}{|\mathcal G_t| \left( |\mathcal G_t|-1 \right)} \sum _{i,l \in \mathcal G_t} \left\| \theta_t^{(i)}-\theta^{(l)}_t \right\|^2 \leq \mathcal E^{\Delta}_t.
\end{align*}
\end{proof}

Next, we want to see how good our \emph{realized} average gradient estimates $\bar{\hat{v}}_t$ are, i.e., how much they deviate from $\nabla J \left( \bar{\theta}_t \right)$ which is what we aim to estimate. Note that this is a key component since both the aggregation and averaging steps introduce bias that we need to control here.
\begin{lemma}
\label{lem:DecByzPG-distortion-aux}
In the setting of Theorem \ref{thm:DecByzPG}, the following bounds hold:
\begin{align*}
\mathcal T_{x1,t}^{(1)} &:= \E \left[ \left\| \bar{\hat{v}}_t-\bar{v}_t \right\|^2 \mid c_{t-1}=1 \right] \leq \frac{C_{avg}^2}{\eta^2}\, \left[ 2 \E[\mathcal E^\Delta_t] + \frac{10 \eta^2C_{ra}\alpha}{\bar{\epsilon}}\widetilde{\mathcal T}_{1,t} \right], \\
\mathcal T_{x0,t}^{(1)} &:= \E \left[ \left\| \bar{\hat{v}}_t-\bar{v}_t \right\|^2 \mid c_{t-1}=0 \right] \leq \frac{C_{avg}^2}{\eta^2}\, \left[ 2 \E[\mathcal E^\Delta_t] + \frac{10 \eta^2 C_{ra}\alpha}{\bar{\epsilon}} \widetilde{\mathcal T}_{0,t} \right], \\
\mathcal T_{1x,t}^{(2)} &:= \E \left[ \left\| \bar{v}_t-\bar{\tilde{v}}_t \right\|^2 \mid c_t=1 \right] \leq C_{ra}\alpha \widetilde{\mathcal T}_{1,t}, \\
\mathcal T_{0x,t}^{(2)} &:= \E \left[ \left\| \bar{v}_t-\bar{\tilde{v}}_t \right\|^2 \mid c_t=0 \right] \leq C_{ra}\alpha \widetilde{\mathcal T}_{0,t}, \\
\mathcal T_{1x,t}^{(3)} &:= \E \left[ \left\| \bar{\tilde{v}}_t-\frac{1}{|\mathcal G_t|} \sum_{i \in \mathcal G_t}  \nabla J \left( \theta_t^{(i)} \right) \right\|^2 \mid c_t=1 \right] \leq \frac{4\sigma^2}{KN}, \quad\text{and}, \\
\mathcal T_{0x,t}^{(3)} &:= \E \left[ \left\| \bar{\tilde{v}}_t - \frac{1}{|\mathcal G_t|} \sum_{i \in \mathcal G_t}  \nabla J \left( \theta_t^{(i)} \right) \right\|^2 \mid c_t=0 \right] \\ &\leq \E \left[ \left\| \bar{\hat{v}}_{t-1}-\frac{1}{|\mathcal G_t|} \sum_{i \in \mathcal G_t}  \nabla J \left( \theta_{t-1}^{(i)} \right) \right\|^2 \right] + \left( \frac{8 C_p}{BK} + \frac{8L^2}{K} \right) \E \left[ \left\| \bar{\theta}_t - \bar{\theta}_{t-1} \right\|^2 \right]
\end{align*}
where $\mathcal T_{ab,t}^{(i)}$ denotes the error in iteration $t$ with $c_t=a, c_{t-1}=b$, and due to averaging agreement for $i=1$, due to aggregation for $i=2$, and due to the randomness of sampling for $i=3$.
\end{lemma}
\begin{proof}
We proceed by showing each of the claimed bounds separately, as follows:
\begin{enumerate}
\item We have
\begin{align*}
\mathcal T_{x1,t}^{(1)}
&= \E \left[ \left\| \bar{\hat{v}}_t-\bar{v}_t \right\|^2 \mid c_{t-1}=1 \right] \\
&= \E \left[ \left\| \frac{1}{\eta} \left( \bar{\theta}_{t+1}-\bar{\theta}_{t} \right) - \frac{1}{\eta} \left( \bar{\tilde{\theta}}_{t+1}-\bar{\theta}_{t} \right) \right\|^2 \mid c_{t-1}=1 \right] \\
&= \frac{1}{\eta^2} \E \left[ \left\| \bar{\theta}_{t+1}-\bar{\tilde{\theta}}_{t+1} \right\|^2 \mid c_{t-1}=1 \right] \\
&\stackrel{(a)}{\leq} \frac{C_{avg}^2}{\eta^2}\, \E \left[ \Delta_2 \left( \langle \tilde{\theta}_t^{(i)} \rangle_{i \in \mathcal G_t} \right)^2 \mid c_{t-1}=1 \right] \\
&\stackrel{(b)}{\leq} \frac{C_{avg}^2}{\eta^2}\, \left[ 2 \E[\mathcal E^\Delta_t] + \frac{10 \eta^2}{\bar{\epsilon}} \widetilde{\mathcal T}_{1,t} \right]
\end{align*}
where (a) is by the definition of averaging agreement, i.e., Definition \ref{def:agree}, and (b) follows from Lemma \ref{lem:diam}.

\item This follows equivalently to 1. --- conditioning the expectation on $c_{t-1}=0$ instead of $c_{t-1}=1$ replaces $\widetilde{\mathcal T}_{1,t} $ by $\widetilde{\mathcal T}_{0,t}$ in the result.
\item First, by applying Lemma \ref{lem:avg-vec-bound} and Definition \ref{def:aggregate},
\begin{align*}
\mathcal T_{1x,t}^{(2)} = \E \left[ \left\| \bar{v}_t-\bar{\tilde{v}}_t \right\|^2 \mid c_t=1 \right]
&\leq \frac{1}{|\mathcal G_t|} \sum_{j \in \mathcal G_t} \E \left[ \left\| v^{(j)}_t-\bar{\tilde{v}}_t \right\|^2 \mid c_t=1 \right] \\
&\leq \frac{1}{|\mathcal G_t|} \sum_{j \in \mathcal G_t} \frac{C_{ra} \alpha}{|\mathcal G_t|(|\mathcal G_t|-1)} \sum_{i,l \in \mathcal G_t} \E \left[ \left\| \tilde{v}_t^{(i)}-\tilde{v}_t^{(l)} \right\|^2 \mid c_t=1 \right] \\
&= \frac{C_{ra} \alpha}{|\mathcal G_t|(|\mathcal G_t|-1)} \sum_{i,l \in \mathcal G_t} \E \left[ \left\| \tilde{v}_t^{(i)}-\tilde{v}_t^{(l)} \right\|^2 \mid c_t=1 \right] \\
&\leq C_{ra}\alpha \widetilde{\mathcal T}_{1,t}
\end{align*}
where the final step is the result of Lemma \ref{lem:DecByzPG-variance-bound}.
\item Again, this is equivalent to 3. --- conditioning the expectation on $c_{t}=0$ instead of $c_{t}=1$ replaces $\widetilde{\mathcal T}_{1,t} $ by $\widetilde{\mathcal T}_{0,t}$ in the result.
\item By plugging in the definition of $\bar{\tilde{v}}_t$, we get
\begin{align*}
\mathcal T_{1x,t}^{(3)} &= \E \left[ \left\| \bar{\tilde{v}}_t- \frac{1}{|\mathcal G_t|} \sum_{i \in \mathcal G_t} \nabla J \left( \theta^{(i)}_t \right) \right\|^2 \mid c_t=1 \right] \\
&\leq \E \left[ \left\| \frac{1}{|\mathcal G_t|} \sum_{i \in \mathcal G_t} \frac{1}{N} \sum_{j=1}^N g \left( \tau_{t,j}^{(i)} \mid \theta_t^{(i)} \right) - \frac{1}{|\mathcal G_t|} \sum_{i \in \mathcal G_t} \nabla J \left( \theta^{(i)}_t \right) \right\|^2 \right] \\
&= \E \left[ \left\| \frac{1}{|\mathcal G_t|N} \sum_{i \in \mathcal G_t} \sum_{j=1}^N g \left( \tau_{t,j}^{(i)} \mid \theta_t^{(i)} \right) - \nabla J \left( \theta^{(i)}_t \right) \right\|^2 \right] \\
&\leq \frac{\sigma^2}{|\mathcal G_t|N} \\
&\leq \frac{4\sigma^2}{KN}
\end{align*}
where the second to last step follows by an argument similar to Lemma \ref{lem:indep-sampling-variance}. The last step then is due to $|\mathcal G_t| \geq (1-\bar\alpha)K \geq (1-\alpha_{\max})K \geq K/4$.
\item First, observe that $\hat{\Delta}^M \left( \theta_t^{(k)},\theta_{t-1}^{(k)} \right)$ is an unbiased estimate of $\Delta \left( \theta_t^{(k)},\theta_{t-1}^{(k)} \right)$.
Then,
\begin{align*}
\mathcal T_{0x,t}^{(3)} &= \E \left[ \left\| \bar{\tilde{v}}_t - \frac{1}{|\mathcal G_t|} \sum_{i \in \mathcal G_t}  \nabla J \left( \theta_t^{(i)} \right) \right\|^2 \mid c_t=0 \right] \\
&= \E \left[ \left\| \frac{1}{|\mathcal G_t|} \sum_{i \in \mathcal G_t} \hat{\Delta}_i^B \left( \theta_t^{(i)},\theta_{t-1}^{(i)} \right) + \hat{v}_{t-1}^{(i)} - \nabla J \left( \theta_t^{(i)} \right) \right\|^2 \right] \\
&= \E \left[ \left\| \frac{1}{|\mathcal G_t|} \sum_{i \in \mathcal G_t} \hat{\Delta}_i^B \left( \theta_t^{(i)},\theta_{t-1}^{(i)} \right) + \hat{v}_{t-1}^{(i)} + \nabla J \left( \theta_{t-1}^{(i)} \right) - \nabla J \left( \theta_{t-1}^{(i)} \right) - \nabla J \left( \theta_t^{(i)} \right) \right\|^2 \right] \\
&= \E \Biggr[ \Biggr\| \underbrace{\Bigg( \bar{\hat{v}}_{t-1}- \frac{1}{|\mathcal G_t|} \sum_{i \in \mathcal G_t} \nabla J \left( \theta^{(i)}_{t-1} \right) \Bigg)}_{=: X_1} + \underbrace{\Bigg( \frac{1}{|\mathcal G_t|} \sum_{i \in \mathcal G_t} \hat{\Delta}_i^B \left( \theta_t^{(i)},\theta_{t-1}^{(i)} \right) - \Delta_i \left( \theta_t^{(i)},\theta_{t-1}^{(i)} \right) \Bigg)}_{=:X_2} \Biggr\|^2 \Biggr] \\
&\stackrel{(a)}{=} \E \left[ \left\| \bar{\hat{v}}_{t-1}- \frac{1}{|\mathcal G_t|} \sum_{i \in \mathcal G_t} \nabla J \left( \theta^{(i)}_{t-1} \right) \right\|^2 \right] + \E \left[ \left\| \frac{1}{|\mathcal G_t|} \sum_{i \in \mathcal G_t} \hat{\Delta}_i^B \left( \theta_t^{(i)},\theta_{t-1}^{(i)} \right) - \Delta_i \left( \theta_t^{(i)},\theta_{t-1}^{(i)} \right) \right\|^2 \right] \\
&\stackrel{(b)}{\leq} \E \left[ \left\| \bar{\hat{v}}_{t-1}- \frac{1}{|\mathcal G_t|} \sum_{i \in \mathcal G_t} \nabla J \left( \theta^{(i)}_{t-1} \right) \right\|^2 \right] + \frac{4}{K^2} \sum_{i \in \mathcal G_t} \E \left[ \left\| \hat{\Delta}_i^B \left( \theta_t^{(i)},\theta_{t-1}^{(i)} \right) - \Delta_i \left( \theta_t^{(i)},\theta_{t-1}^{(i)} \right) \right\|^2 \right] \\
&\stackrel{(c)}{\leq} \E \left[ \left\| \bar{\hat{v}}_{t-1}- \frac{1}{|\mathcal G_t|} \sum_{i \in \mathcal G_t} \nabla J \left( \theta^{(i)}_{t-1} \right) \right\|^2 \right] + \left( \frac{8 C_p}{BK} + \frac{8L^2}{K} \right) \E \left[ \left\| \bar{\theta}_t - \bar{\theta}_{t-1} \right\|^2 \right]
\end{align*}

Step (a) requires further justification: Note that $X_1$ does not depend on any randomness at step $t$, and also $\E \left[ X_2 \right]=0$. Hence, denoting by $\E_t$ the expectation over the randomness in step $t$, we can use the tower rule to get
\begin{align*}
\E \left[ \left\| X_1+X_2 \right\|^2 \right] &= \E \left[ \E_t \left[ \left\| X_1+X_2 \right\|^2 \right] \right] \\
&= \E \Bigg[ \left\| X_1 \right\|^2 + 2 \langle X_1, \underbrace{\E_t \left[ X_2 \right]}_{=0} \rangle + \E_t \left[ \left\| X_2 \right\|^2 \right] \Bigg] \\[4pt]
&= \E \left[ \left\| X_1 \right\|^2 \right]  + \E \left[ \left\| X_2 \right\|^2 \right]
\end{align*}
which then justifies (a). Step (b) follows from a similar argument on a sum of $|\mathcal G_t|$ random variables, the fact that $\E \left[ \hat{\Delta}^M \left( \theta_t^{(k)},\theta_{t-1}^{(k)} \right) \right] = \Delta \left( \theta_t^{(k)},\theta_{t-1}^{(k)} \right)$, and $|\mathcal G_t| \geq K/4$ as noted above. Step (c) then follows from the same argument used in the second part of the proof of Lemma \ref{lem:DecByzPG-variance-bound}.
\end{enumerate}
\end{proof}

\begin{lemma}\label{lem:DecByzPG-distortion}
In the setting of Lemma \ref{lem:DecByzPG-distortion-aux}, let $\mathcal T_t:=\left\| \bar{\hat{v}}_{t}-\frac{1}{|\mathcal G_t|} \sum_{i \in \mathcal G_t} \nabla J \left( \theta^{(i)}_t \right) \right\|^2$. Then, we have
\begin{align*}
\E \left[ \left\| \mathcal T_t \right\|^2 \right] &\leq \left( 1-\frac{p}{2} \right) \E \left[ \left\| \mathcal T_{t-1} \right\|^2 \right] + \frac{Ap}{4} \E \left[ \left\| \bar{\theta}_t - \bar{\theta}_{t-1} \right\|^2 \right] + \frac{C_1\sigma^2p\alpha}{4N} + \frac{C_2\sigma^2p}{4KN}
\end{align*}
where $C_1,C_2>0$ are constants, and $A = \Theta \left( \frac{\alpha}{p^2B} + \frac{1}{pK} \right)$.
\end{lemma}
\begin{proof}
By the law of total expectation,
\begin{align*}
\E \left[ \mathcal T_t \right] &= p \E \left[ \mathcal T_t \mid c_t=1 \right]+ (1-p)\E \left[ \mathcal T_t \mid c_t=0 \right] \\
&= p\E \left[ \mathcal T_t \mid c_t=1 \right] + (1-p)\Big( p\E \left[ \mathcal T_t \mid c_t=0 \land c_{t-1}=1 \right] + (1-p) \E \left[ \mathcal T_t \mid c_t=0 \land c_{t-1}=0 \right] \Big).
\end{align*}
By applying (\ref{eqn:beta-norm-sum}) and (\ref{eqn:norm-of-sum}) of Lemma \ref{lem:basics}, we further choose the following way of bounding:
\begin{align*}
\E \left[ \mathcal T_t \mid c_t=1 \right] &\leq 3 \left( p \mathcal T_{x1,t}^{(1)} + (1-p) \mathcal T_{x0,t}^{(1)} \right) + 3 \mathcal T_{1x,t}^{(2)} + 3 \mathcal T_{1x,t}^{(3)}, \\
\E \left[ \mathcal T_t \mid c_t=0 \land c_{t-1}=1 \ \right] &\leq 22 \mathcal T_{x1,t}^{(1)} + 22 \mathcal T_{0x,t}^{(2)} + 1.1 \mathcal T_{0x,t}^{(3)}, \quad\text{and} \\
\E \left[ \mathcal T_t \mid c_t=0 \land c_{t-1}=0 \right] &\leq \left( 2+\frac{8}{p} \right) \mathcal T_{x0,t}^{(1)} + \left( 2+\frac{8}{p} \right) \mathcal T_{0x,t}^{(2)} + \left( 1+\frac{p}{4} \right) \mathcal T_{0x,t}^{(3)}.
\end{align*}
Putting everything together, we get
\begin{align*}
\E \left[ \mathcal T_t \right] &\leq \left( 3p^2+22p \left( 1-p \right) \right) \mathcal T_{x1,t}^{(1)} + \left( 3p \left( 1-p \right) + \left( 1-p \right)^2 \left( 2+\frac{8}{p} \right) \right) \mathcal T_{x0,t}^{(1)} + 3p \mathcal T_{1x,t}^{(2)} \\
&\;+ \left( 22p \left( 1-p \right) + \left( 1-p \right)^2 \left( 2+\frac{8}{p} \right) \right) \mathcal T_{0x,t}^{(2)} + 3p \mathcal T_{1x,t}^{(3)} \\&+ \left( 1.1p \left( 1-p \right) + \left( 1-p \right)^2 \left( 1+\frac{p}{4} \right) \right) \mathcal T_{0x,t}^{(3)} \\
&\leq 22p \mathcal T_{x1,t}^{(1)} + \frac{8}{p} \mathcal T_{x0,t}^{(1)} + 3p \mathcal T_{1x,t}^{(2)} + \frac{22}{p} \mathcal T_{0x,t}^{(2)} + 3p \mathcal T_{1x,t}^{(3)} + \left( 1-\frac{p}{2} \right) \mathcal T_{0x,t}^{(3)}
\end{align*}
where the last step follows from Lemma \ref{lem:basics}, (\ref{eqn:p-prod-2}).

After plugging in the results of Lemma \ref{lem:DecByzPG-distortion-aux}, rearranging, and choosing large enough constants $C_1,C_2>0$, we get
\begin{align*}
&\E \left[ \mathcal T_t \right] \\[4pt]
\leq &\;\widetilde{\mathcal T}_{1,t} \left[ \frac{220p c \alpha C_{avg}^2}{\bar{\epsilon}}+3pC_{ra}\alpha \right] + \widetilde{\mathcal T}_{0,t} \left[ \frac{80C_{ra}\alpha C_{avg}^2}{p \bar{\epsilon}}+\frac{22C_{ra}\alpha}{p} \right] + \frac{4C_{avg}^2\overline{\mathcal E}^{\Delta,\kappa}}{\eta^2}\left( 11p+\frac{4}{p} \right) \\
&\quad+ \frac{12p\sigma^2}{KN} + \left( 1-\frac{p}{2} \right) \mathcal T_{0x,t}^{(3)}\\[6pt]
= &\;\overline{\mathcal E}^{\Delta,\kappa} \cdot \underbrace{\left[ \frac{4C_{avg}^2}{\eta^2}\left( 11p+\frac{4}{p} \right)+8L^2 \left( \frac{220p c \alpha C_{avg}^2}{\bar{\epsilon}}+3pC_{ra}\alpha \right) + \left( \frac{12}{\eta^2} + \frac{36 C_p}{B} \right) \left( \frac{80C_{ra}\alpha C_{avg}^2}{p \bar{\epsilon}}+\frac{22C_{ra}\alpha}{p} \right) \right]}_{=: \mathcal Z}\\
&\quad + \frac{8\sigma^2}{N} \left( \frac{220p c \alpha C_{avg}^2}{\bar{\epsilon}}+3pC_{ra}\alpha + \frac{3p}{2K} \right) + \frac{18 C_p}{B} \E\left\| \bar{\theta}_t - \bar{\theta}_{t-1} \right\|^2 \left( \frac{80C_{ra}\alpha C_{avg}^2}{p \bar{\epsilon}}+\frac{22C_{ra}\alpha}{p} \right)\\
&\quad+ \left( 1-\frac{p}{2} \right) \left[ \E \left[ \mathcal T_{t-1} \right] + \left( \frac{8 C_p}{BK} + \frac{8L^2}{K} \right) \E \left[ \left\| \bar{\theta}_t - \bar{\theta}_{t-1} \right\|^2 \right] \right] \\
\stackrel{(a)}{\leq} &\;\frac{\sigma^2}{4N} \left( C_1 p \alpha + \frac{C_2 p}{K} \right) + \mathcal O \left( \frac{\alpha}{p} \right) \E\left\| \bar{\theta}_t - \bar{\theta}_{t-1} \right\|^2 + \left( 1-\frac{p}{2} \right) \left[ \E \left[ \mathcal T_{t-1} \right] + \mathcal O \left( \frac{1}{K} \right) \E \left[ \left\| \bar{\theta}_t - \bar{\theta}_{t-1} \right\|^2 \right] \right].
\end{align*}
In step (a) we switch to asymptotic analysis, treating $\sigma,C_{avg},\eta,\bar{\epsilon},L,C_{ra}$ and $C_p$ as constants, and using $\alpha,p \leq 1$, and $B \geq 1$. Moreover, because of Lemma \ref{lem:diam-error}, we know $\overline{\mathcal E}^{\Delta,\kappa} \leq \mathcal O \left( 2^{-\kappa} \right)$, and further, one can see that $\mathcal Z = \Theta \left( 1/p \right)$. Therefore, with our choice of $\kappa=\Theta \left( \log\frac{KN}{p^2} \right)$, we have $\overline{\mathcal E}^{\Delta,\kappa} \cdot \mathcal Z \leq \mathcal O \left( \frac{p}{KN} \right)$ and this term hence gets swallowed by the constant $C_2$. From the last line above, it becomes clear that for
\begin{align*}
A = \Theta \left( \frac{\alpha}{p^2B} + \frac{1}{pK} \right)
\end{align*}
the result of the lemma follows.
\end{proof}

\section{Experimental setup}
\label{app:experiment-setup}

\subsection{Hyperparameters}
Throughout all experiments, policies are parameterized as neural networks, and we use the Adam \cite{kingma2014adam} optimizer during training. We summarize the chosen hyperparameters, some of which are adopted from \cite{fan2021fault} and \cite{gargiani2022page}, in Table \ref{tab:params}, and refer to the repository provided as supplementary material for instructions on how to run the script that specifies the seeds and reproduces the experiments contained in this paper.

\subsection{Compute Requirements}
Our experiments have been conducted on a 4-core \emph{Intel(R) Xeon(R) CPU E3-1284L v4} clocked at 2.90GHz and equipped with 8Gbs of memory. Since no GPUs have been used, all experiments are fully reproducible using the specified seeds. Moreover, each individual experiment can terminate within less than 5 hours for CartPole tasks, and less than 15 hours for LunarLander tasks, on the specified hardware.

\begin{table}[ht]
\centering
\begin{tabular}{@{}ccc@{}}
\hline
Hyperparameters & CartPole & LunarLander \\ \hline
NN policy & Categorical MLP & Categorical MLP \\
NN hidden weights & 16,16 & 64,64 \\
NN activation & ReLU & Tanh \\
NN output activation & Tanh & Tanh \\
Step size (Adam) $\eta$ & 5e-3 & 1e-3 \\
Discount factor $\gamma$ & 0.999 & 0.999 \\
Task horizon $H$ & 500 & 1000 \\
Small batch size $B$ & 4 & 32 \\
Large batch size $N$ & 50 & 96 \\
Switching probability $p$ & 0.2 & 0.2
\end{tabular}
\bigskip
\caption{Hyperparameters used in our experiments.}
\label{tab:params}
\end{table}

\section{Additional experiments}
\label{app:add-experiments}

\subsection{Experiments for \textsc{ByzPG}}

\begin{figure}[h]
\centering
\includegraphics[width=.55\linewidth]{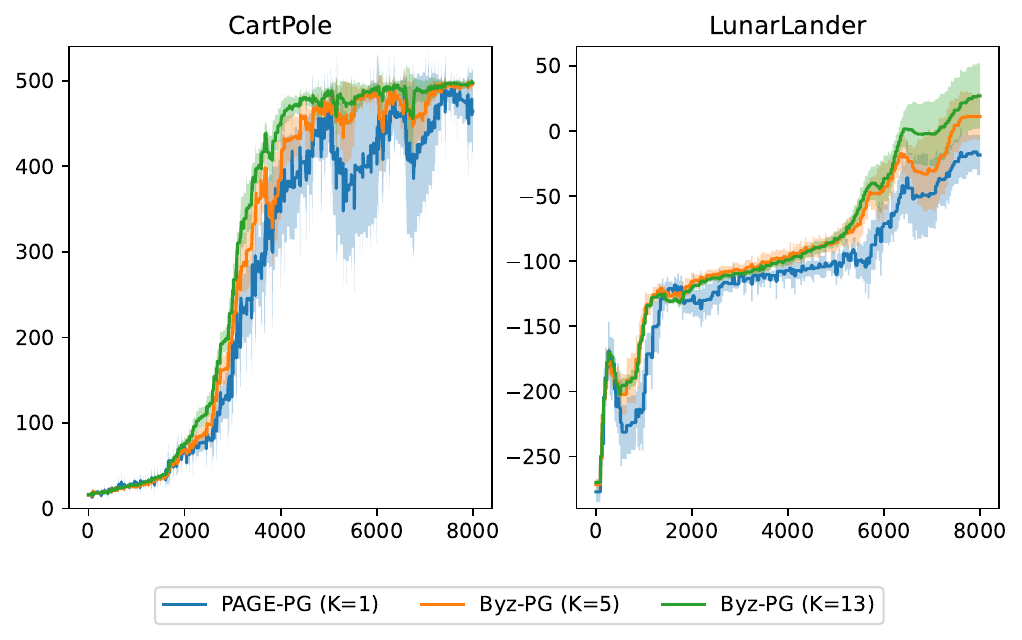}
\caption{Performance of \textsc{ByzPG} for different federation sizes when all agents behave honestly (i.e.\ $\alpha=0$).}
\Description{Performance of \textsc{ByzPG} for different federation sizes when all agents behave honestly (i.e.\ $\alpha=0$).}
\label{fig:speedup-byzPG}
\end{figure}
In Figure \ref{fig:speedup-byzPG}, we consider \textsc{ByzPG} in the case $\alpha=0$, with $K=1$ (which is equivalent to \textsc{PAGE-PG} \cite{gargiani2022page}), $K=5$, and $K=13$. Speed-up with increasing number of agents is observable in both environments, as suggested by Corollary \ref{cor:ByzPG}. Such faster convergence provides empirical evidence motivating agents to join a (centralized) federation.

In Figure \ref{fig:cartpole-attack-byzPG} and \ref{fig:lunar-attack-byzPG}, we compare \textsc{ByzPG} under the same attacks as described in Section \ref{sec:exp} to (a) \textsc{PAGE-PG} \cite{gargiani2022page}, the SOTA single-agent PG method that \textsc{ByzPG} reduces to when $K=1$, and (b) \textsc{Fed-PAGE-PG}, a naive centralized federated (but not fault-tolerant) version of \textsc{PAGE-PG} where aggregation of gradients is done by averaging.

\begin{figure}[H]
\centering
\includegraphics[width=0.9\linewidth]{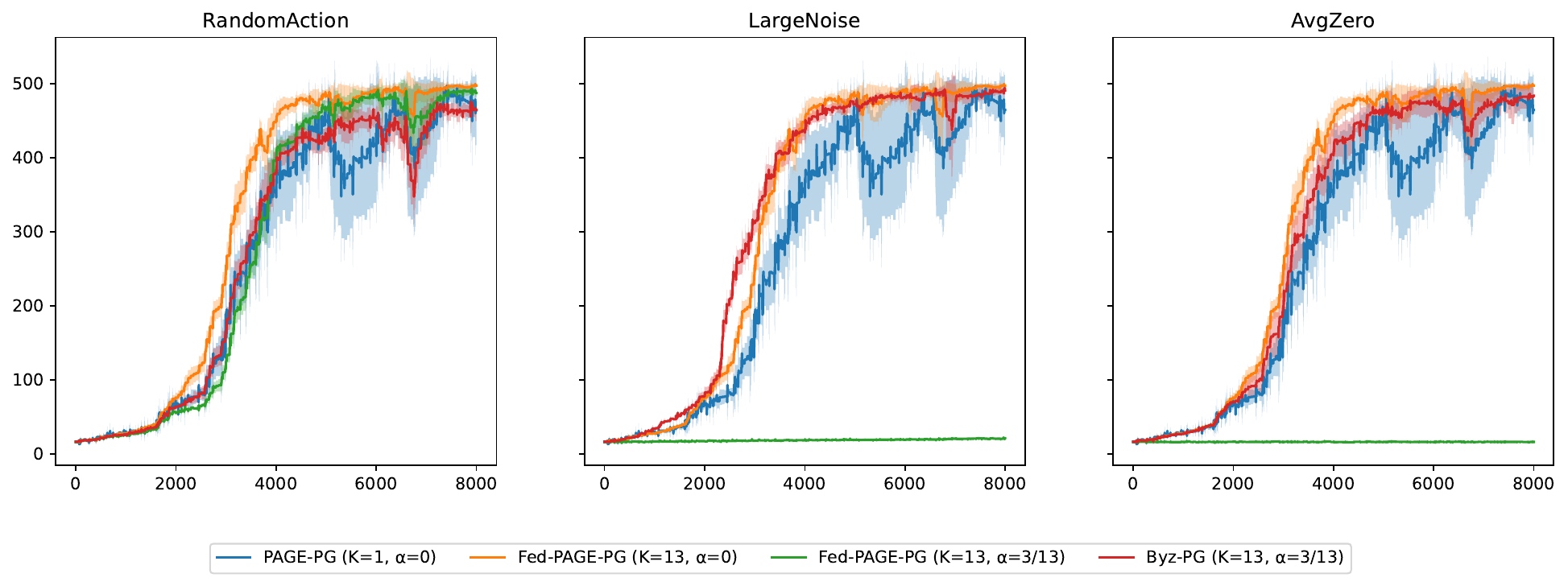}
\caption{Performance \& resilience of \textsc{ByzPG} for CartPole w.r.t.\ our three attack types.}
\label{fig:cartpole-attack-byzPG}
\Description{Performance \& resilience of \textsc{ByzPG} for CartPole w.r.t.\ our three attack types.}
\end{figure}

\begin{figure}[H]
\centering
\includegraphics[width=0.9\linewidth]{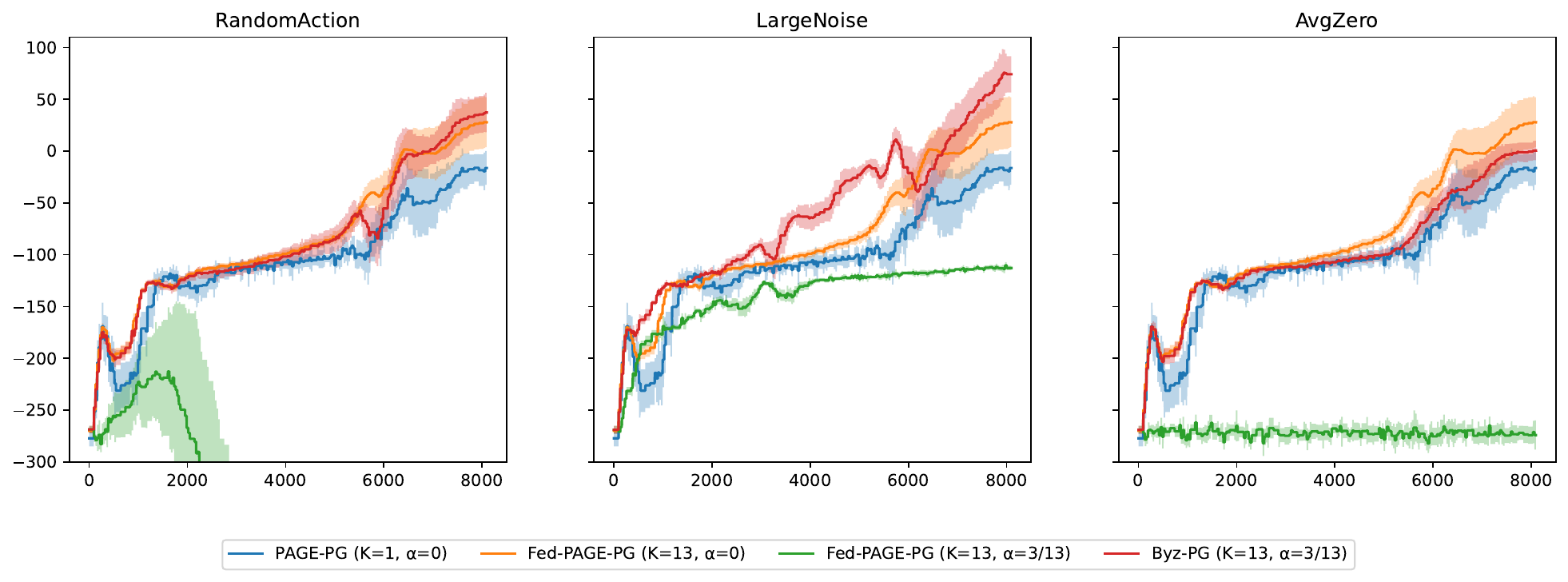}
\caption{Performance \& resilience of \textsc{ByzPG} for LunarLander w.r.t.\ our three attack types.}
\label{fig:lunar-attack-byzPG}
\Description{Performance \& resilience of \textsc{ByzPG} for LunarLander w.r.t.\ our three attack types.}
\end{figure}

We observe that for both environments and all attacks, \textsc{ByzPG} performs nearly on par with the unattacked \textsc{PAGE-PG}. This empirically supports the Byzantine fault-tolerance of \textsc{ByzPG}. Furthermore, for CartPole, as expected, \textbf{LargeNoise} and \textbf{AvgZero} are highly effective against the non-fault-tolerant method, while \textbf{RandomAction} barely shows any effect, similar to our observation for the decentralized case in Section \ref{sec:exp}. For the more difficult task of LunarLander, already \textbf{RandomAction} breaks \textsc{PAGE-PG}. Finally, we conclude by pointing out that in all cases \textsc{ByzPG} with $K=13$ and $\alpha > 0$ outperforms \textsc{PAGE-PG} with $K=1$ (and $\alpha=0$), meaning that in our experiments, despite the presence of Byzantines, joining the (centralized) federation is empirical beneficial for faster convergence.

\end{document}